\documentclass{article}

\usepackage{PRIMEarxiv}

\usepackage[utf8]{inputenc} 
\usepackage[T1]{fontenc}    
\usepackage{hyperref}       
\usepackage{url}            
\usepackage{booktabs}       
\usepackage{amsfonts}       
\usepackage{nicefrac}       
\usepackage{microtype}      
\usepackage{lipsum}
\usepackage{fancyhdr}       
\usepackage{graphicx}       

\usepackage{amssymb}
 \usepackage{amsthm}
 \usepackage{amsmath}
\usepackage{algorithm}
\usepackage{algorithmic}
\usepackage{scalerel}
\usepackage{hyperref}
\usepackage{wrapfig}

\newtheorem{theorem}{Theorem}
\newtheorem{definition}{Definition}
\newtheorem{example}{Example}
\newtheorem{lemma}{Lemma}

\newtheorem{remark}{Remark}

\newcommand{\eqdef}{\overset{def}{=}}

\title{Autoencoders for a manifold learning problem with a Jacobian rank constraint}

\author{
  Rustem Takhanov, Y. Sultan Abylkairov, Maxat Tezekbayev \\
  Department of Mathematics, School of Sciences and Humanities \\
  Nazarbayev University \\
  Astana\\
  \texttt{\{rustem.takhanov, sultan.abylkairov, maxat.tezekbayev\}@nu.edu.kz} \\
}

\begin{document}
\maketitle

\begin{abstract}
We formulate the manifold learning problem as the problem of finding an operator that maps any point to a close neighbor that lies on a ``hidden'' $k$-dimensional manifold. We call this operator the correcting function. Under this formulation, autoencoders can be viewed as a tool to approximate the correcting function. 
Given an autoencoder whose Jacobian has rank $k$, we deduce from the classical Constant Rank Theorem that its range has a structure of a $k$-dimensional manifold.

A $k$-dimensionality of the range can be forced by the architecture of an autoencoder (by fixing the dimension of the code space), or alternatively, by an additional constraint that the rank of the autoencoder mapping is not greater than $k$. This constraint is included in the objective function as a new term, namely a squared Ky-Fan $k$-antinorm of the Jacobian function.
We claim that this constraint is a factor that effectively reduces the dimension of the range of an autoencoder, additionally to the reduction defined by the architecture. We also add a new curvature term into the objective. To conclude, we experimentally compare our approach with the CAE+H method on synthetic and real-world datasets.
\end{abstract}


\section{Introduction} 
The unsupervised manifold hypothesis~\cite{caytonalgorithms} is a general statement about a probability distribution in a high-dimensional space that conjectures data to be locally concentrated along a low dimensional (compared to a dimension of an ambient space) surface. Historically, the hypothesis appeared as a development of Principal Components Analysis (PCA) method and can be considered as a non-linear extension of it. 

If the hypothesis holds for a given dataset, then recovering the underlying manifold is the main goal of the manifold learning problem. Knowledge extracted at a manifold learning stage could be further exploited in various settings such as manifold clustering, manifold alignment or regularization of supervised learning techniques~\cite{Goodfellow-et-al-2016}. Besides reconstructing a manifold, it is of special interest to calculate a basis of the tangent space of a ``hidden'' manifold at any point. Usually, tangent spaces are reconstructed at a pre-training stage, i.e. prior to training a supervised model. Afterward, different techniques, such as the TangentProp method~\cite{TangentProp}, can be applied to regularize a supervised model to make an output of a regression function more robust with respect to an input perturbation along the manifold.

Various techniques, including~\cite{Roweis,Tenenbaum,Donoho5591,Vincent,Diffeomorphic,Belkin} and many others, were suggested for the manifold learning problem. Applications of these techniques include Human Pose Recovery~\cite{Chaoqun1,Chaoqun2}, Fine-Grained Image Recognition~\cite{YuJun}, structured prediction~\cite{YulingFan,ChaoTan} and many others.
One of the most popular approaches is based on contractive autoencoders (CAE)~\cite{Kramer,Hinton}. The contractive autoencoder with a hessian method (CAE+H)~\cite{Rifai} trains an autoencoder by minimizing the reconstruction error and at the same time ensuring the robustness of an encoder's Jacobian with respect to small perturbations of an input. The third term of CAE+H is a standard regularizer of the norm of the Jacobian. Experiments showed that low-dimensional representations of input objects obtained by the CAE+H at the pre-training stage help to improve the accuracy of trained classifiers for many different tasks~\cite{Rifai,Rifai2}.

{\bf Motivation.} Our approach starts from the analysis of mathematical structures behind the idea of the contractive autoencoder.
The fact that the reconstruction of an input vector in CAE goes through an intermediate low-dimensional code space guarantees that the output point belongs to some low-dimensional manifold (this follows from the classical Constant Rank Theorem). The idea of enforcing the robustness of the Jacobian with respect to perturbations comes from the geometric fact that regularizing this term is equivalent to bounding a manifold's curvature. 
We start from this geometric interpretation but suggest a novel approach to training an autoencoder (and, therefore, to choosing an architecture). We view the dimensionality reduction as a result of two factors: the first is defined by an architecture of an autoencoder (the dimension of the code space) and the second is defined by the structure of trained weights of an autoencoder. The dimension of the code space bounds the rank of the autoencoder by construction. 
The second factor is enforced by the addition of a new term to a reconstruction error that forces the Jacobian of an autoencoder to be close to some matrix of rank $k$. We show that this constraint is equivalent to requiring the range of the autoencoder to be close to a $k$-dimensional manifold. Thus, the second dimension reduction mechanism is of a ``soft'' nature, unlike the first mechanism, which is a ``crisp'' dimension reduction defined by the architecture of the autoencoder. 

Why do we need a soft rank-controlling mechanism? Real-world datasets are heterogeneous, namely, they consist of parts that come from different distributions, and supports of these distributions can have different intrinsic dimensions. Let us demonstrate the latter idea with the following toy example. 
\begin{wrapfigure}{l}{0.5\textwidth}\label{toy}
    \centering
    \includegraphics[width=0.5\textwidth]{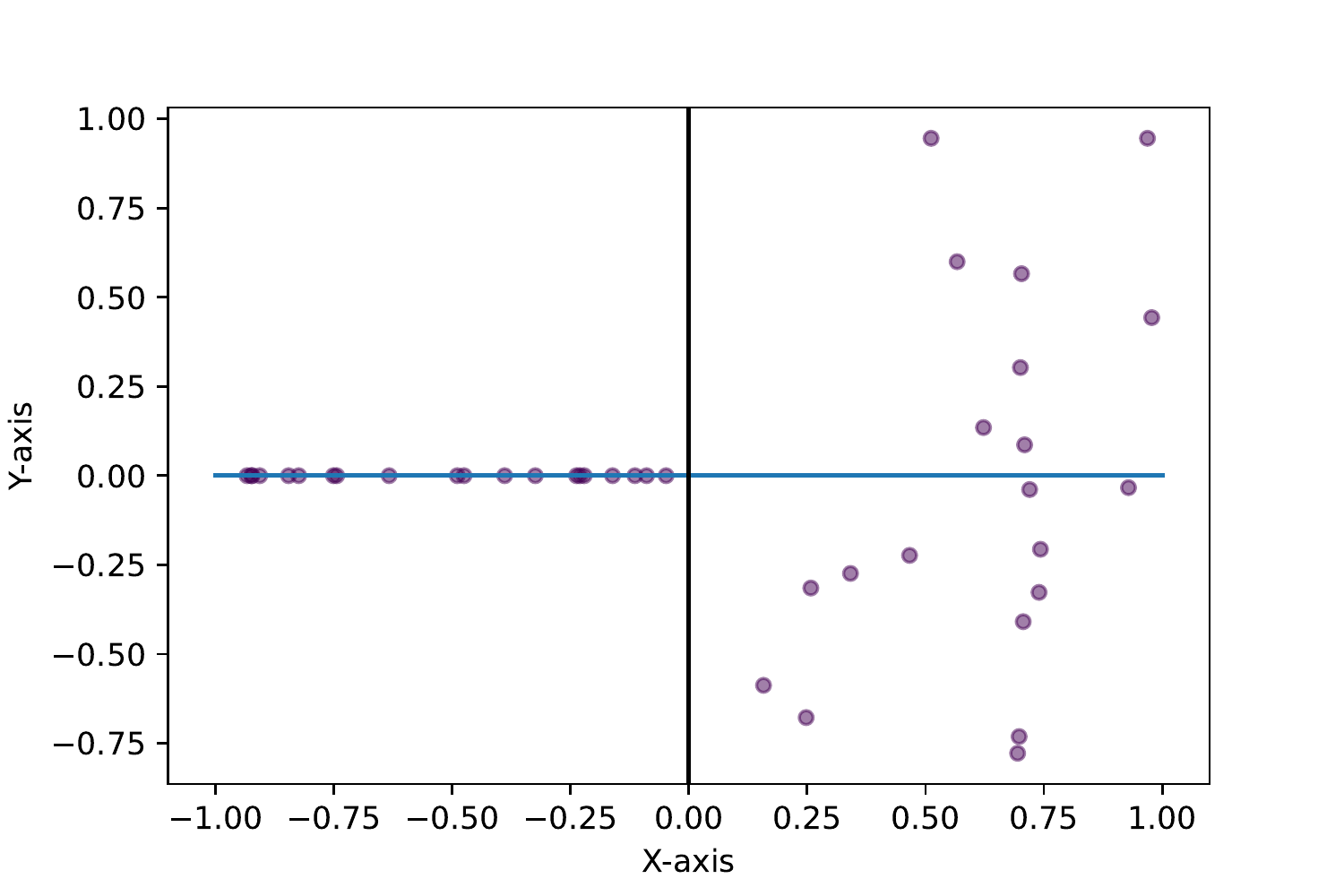}
\caption{An optimal curve (blue line) has the reconstruction error $\frac{1}{6}$ because the dataset consists of 2 parts, 1-dimensional and 2-dimensional.}
\end{wrapfigure}
\begin{example}
Suppose that points $\{{\mathbf x}_i\}_{i=1}^N\subseteq {\mathbb R}^2$ are sampled independently in the following way: ${\mathbf x}_i = (x_i, y_i[x_i>0])$ where $x_i, y_i$ are independent and uniform over $[-1,1]$ and $[\cdot]$ is the Iverson bracket (see Figure~\ref{toy}). 

Let us consider a contractive autoencoder with a dimension of the code space 1, i.e. ${\mathbf g} = {\mathbf d}\circ {\mathbf e}$ where ${\mathbf e}: {\mathbb R}^2\to {\mathbb R}$ is an encoder and ${\mathbf d}: {\mathbb R}\to {\mathbb R}^2$ is a decoder. Here, by construction, the range of ${\mathbf g}$ is a curve. If the autoencoder's architecture captures projections, then a training process will result in an optimal ${\mathbf g}(x,y) \approx (x, 0)$. A reconstruction error in the latter case will be ${\rm RE} = {\mathbb E}\big({\mathbf g}(x,y)-(x,y)\big)^2\approx \frac{1}{4}\int_{-1}^{1}t^2dt = \frac{1}{6}$. Alternatively, if we define the dimension of the code space as 2, i.e. ${\mathbf e}, {\mathbf d}: {\mathbb R}^2\to {\mathbb R}^2$, natural architectures allow us to train ${\mathbf g}(x,y) \approx (x,y [y>0])$. In the latter case, the reconstruction error becomes zero at the expense of losing a desirable property of having a 1-dimensional range of the mapping ${\mathbf g}$. But we gain another desirable property: the gradient of ${\mathbf g}$ has rank 1 in approximately half of the dataset, i.e. ${\rm rank\,}\nabla {\mathbf g}(x,y)=1$ if $y<0$. Thus, we have a tradeoff between two desirable properties: (a) the dimension of the code space 1 guarantees that ${\rm rank\,}\nabla {\mathbf g}(x,y)=1$ everywhere, but it is impossible to achieve a low reconstruction error, (b) by omitting the latter crisp condition, we may achieve lowest reconstruction error possible, but ${\rm rank\,}\nabla {\mathbf g}(x,y)=1$ is now satisfied only at half of the points.
\end{example}
In the given example, the support of the distribution is a union of a line and a half-plane, though one can imagine more sophisticated situations where manifolds of dimensions $n_1, n_2, n_3, \cdots$ each contain a substantial piece of a dataset.  In this case, a crisp requirement of a fixed rank of an autoencoder can be naturally replaced by some soft measure. Our paper suggests one specific way to introduce such a soft measure based on a Ky Fan $k$-antinorm.

The paper is organized in the following way. In Section~\ref{formulation} we define a $k$-dimensional manifold and show that any mapping whose range is a $k$-dimensional manifold has a Jacobian of rank $k$ (Theorem~\ref{primal}). Next, we present a weaker opposite statement, Theorem~\ref{opposite} (using a classical result from topology, the Constant Rank Theorem). We demonstrate that from a full rank mapping onto a $k$-dimensional manifold (i.e. a mapping whose Jacobian is of rank $k$) one can extract full information about the manifold itself (and a local structure of the manifold is defined by the Jacobian, the Hessian, or higher-order derivatives of the mapping). These results allow us to define an objective in the manifold learning task only in terms of such a mapping, which is subject to optimization in the form of an autoencoder. In subsection~\ref{Curvature} we introduce three expressions, $\kappa_0$, $\kappa_1$ and $\kappa_2$,  given in terms of the Jacobian/Hessian of a mapping onto the manifold, that are tightly related to the curvature of the manifold. The curvature-related term is then included in the objective of an autoencoder training procedure. Further, in subsection~\ref{softly}, we describe two main approaches to force a Jacobian rank to be not greater than $k$, a crisp  (the architecture-based) and a soft (based on a Ky-Fan $k$-antinorm).
In Section~\ref{a-algoritm} we describe an alternating algorithm for the minimization of the objective with a soft rank reducing term. In Section~\ref{experiments} we describe our experiments on synthetic and real-world data and proceed to the conclusion.
\subsection{Notation}
Throughout, we will denote components of a vector ${\mathbf x}\in {\mathbb R}^d$ by  $x_1, \cdots, x_d$ and components of a mapping ${\mathbf g}: {\mathbb R}^d\to {\mathbb R}^e$ by $g_1: {\mathbb R}^d\to {\mathbb R}, \cdots, g_e: {\mathbb R}^d\to {\mathbb R}$.  The Jacobian matrix $[\frac{\partial g_i}{\partial x_j}]_{i=\overline{1,d},j=\overline{1,e}}$ is denoted by $J_{\mathbf g}$. For $\Omega\subseteq {\mathbb R}^d$, ${\mathbf g}(\Omega) = \{{\mathbf g}(x) \mid x\in \Omega\}$.
For a function $f:{\mathbb R}^n \to {\mathbb R}\cup \{+\infty\}$, let us denote $$\sup \{L \mid   L=\lim_{i\to \infty}f(\boldsymbol{\varepsilon}_i)<\infty, \lim_{i\to \infty}\boldsymbol{\varepsilon}_i={\mathbf 0}\}$$ by $\overline{\lim}_{\boldsymbol{\varepsilon}\to {\mathbf 0}} f(\boldsymbol{\varepsilon})$. If $f$ is twice differentiable at a point ${\mathbf x}$, then a Hessian of $f$ at ${\mathbf x}$ is denoted by $H_f({\mathbf x})$.
Also, we denote the euclidean norm, the maximum norm and the Frobenius norm by $\|{\mathbf x}\|$, $\|{\mathbf x}\|_\infty $ and $\|A\|_F$ correspondingly. The composition of functions ${\mathbf d}$ and ${\mathbf e}$ is denoted by ${\mathbf d}\circ {\mathbf e}$. The symbol ${\mathcal O}(f(n))$ denotes any function bounded by $Cf(n)$ for some constant $C>0$.
\section{Manifold learning problem}\label{formulation}
In the mathematical literature one can find many non-equivalent definitions of manifold. 
We prefer more elementary language, avoiding such terms as the chart, the atlas, the diffeomorphism, etc. With a little abuse of terminology, we give definitions that are general enough to serve our purposes.

\begin{definition}\label{like}
A subset $U\subseteq {\mathbb R}^n$ is said to be {\em like ${\mathbb R}^k$} if there exist  an open set $V\subseteq {\mathbb R}^k$ and infinitely differentiable one-to-one mapping $\phi: V\rightarrow U$, whose inverse $\phi^{-1}: U\rightarrow V$ can be extended to an infinitely differentiable mapping $\psi: S\to {\mathbb R}^k$ on an open superset $S\supseteq U$, i.e. $\phi^{-1}(x) = \psi(x), x\in U$.
\end{definition}

\begin{definition}
A subset $\mathcal{M}\subseteq {\mathbb R}^n$ is called a $k$-dimensional manifold (or, a $k$-manifold) if any point 
$p\in \mathcal{M}$ has an open neighborhood $U\subseteq {\mathbb R}^n$ such that $U\cap \mathcal{M}$ is like ${\mathbb R}^k$. 
\end{definition}

Nonetheless, let us note that a given definition is equivalent to a standard notion of a $k$-dimensional smooth manifold embedded in ${\mathbb R}^n$~\cite{lee2013introduction}. The strong Whitney embedding theorem tells us that actually this notion captures all smooth manifolds modulo a diffeomorphism.

The (unsupervised) manifold hypothesis claims that observable data points ${\mathbf x}_1$, ${\mathbf x}_2$, $\cdots$, ${\mathbf x}_N$ can be slightly ``corrected'', i.e. ${\mathbf x}_i$ can be substituted with ${\mathbf y}_i$, such that the cost of the correction:
$$
\sum_{i=1}^N \|{\mathbf x}_i-{\mathbf y}_i\|^2
$$
is small and all our substitutes ${\mathbf y}_1$, ${\mathbf y}_2$, $\cdots$, ${\mathbf y}_N$ lie on some $k$-dimensional manifold $\mathcal{M}$. The main assumption of the hypothesis is that $k \ll n$ and the manifold $\mathcal{M}$ is sufficiently smooth (the latter is understood geometrically and discussed in Section~\ref{Curvature}). Without the smoothness assumption, our hypothesis can be trivially satisfied by setting ${\mathbf y}_i\leftarrow {\mathbf x}_i$ and any $k$-manifold $\mathcal{M}$ that contains the initial data points. Obviously, this leads to overfitting.

\subsection{The correcting mapping and constraints on the rank of the Jacobian} 
Assuming the manifold hypothesis, we can formulate the manifold learning problem as the problem of recovering the ``hidden'' manifold $\mathcal{M}$ from data points. If such hidden manifold exists, then there exists a smooth mapping ${\mathbf g}^\ast: {\mathbb R}^n\rightarrow \mathcal{M}$ such that ${\mathbf g}^\ast({\mathbf x}_i) = {\mathbf y}_i$. Let us call ${\mathbf g}^\ast$ the correcting mapping. Given a family of $k$-manifolds ${\mathfrak M} \subseteq \{\mathcal{M} | \mathcal{M}\subseteq {\mathbb R}^n\}$, it is natural to recover the ``hidden'' manifold $\mathcal{M}^\ast$ and the function ${\mathbf g}^\ast$ simply by setting:
\begin{equation}\label{solution}
(\mathcal{M}^\ast, {\mathbf g}^\ast) \leftarrow \arg\min_{\mathcal{M}\in {\mathfrak M}}\min\limits_{{\mathbf g}: {\mathbb R}^n\rightarrow \mathcal{M}} \frac{1}{N}\sum_{i=1}^N \|{\mathbf x}_i-{\mathbf g}({\mathbf x}_i)\|^2,
\end{equation}
where the internal minimization is made over all functions ${\mathbf g}: {\mathbb R}^n\rightarrow \mathcal{M}$. By construction, for fixed $\mathcal{M}$, the minimum over ${\mathbf g}: {\mathbb R}^n\rightarrow \mathcal{M}$ is attained at $${\mathbf g}^\mathcal{M}({\mathbf x}) \in {\rm Arg}\min_{{\mathbf y}\in \mathcal{M}} \|{\mathbf y}-{\mathbf x}\|.$$ 
If the latter argument is not unique then one can choose an arbitrary point from a set of alternatives.
Thus, given the first element in the pair~\eqref{solution}, i.e. the manifold $\mathcal{M}^\ast$, we can also recover ${\mathbf g}^\ast$ by setting $ {\mathbf g}^\ast \leftarrow {\mathbf g}^{\mathcal{M}^\ast}$.

Since the manifold is never given, we have to calculate both the manifold $\mathcal{M}^\ast$ and the correcting mapping ${\mathbf g}^\ast$ simultaneously. 
The following theorem restricts the set of possible correcting mappings, even if $\mathcal{M}^\ast$ is not given.

\begin{theorem}\label{primal} Let ${\mathbf g}: {\mathbb R}^n \rightarrow \mathcal{M}$ be a continuously differentiable mapping such that $\mathcal{M}\subseteq {\mathbb R}^n$ is a $k$-manifold. Then, for any point ${\mathbf x}$, we have ${\rm rank\,} J_{{\mathbf g}}({\mathbf x}) \leq k$.
\end{theorem}
\begin{proof} Since $\mathcal{M}$ is a $k$-manifold, $U\cap \mathcal{M}$ is like ${\mathbb R}^k$ for some open neighbourhood $U$ of ${\mathbf g}({\mathbf x}^\ast)\in \mathcal{M}$. Thus, there exist open sets $S\supseteq U\cap \mathcal{M}$, $V\subseteq {\mathbb R}^k$ and infinitely differentiable mappings $\psi: S\rightarrow {\mathbb R}^k$, $\phi: V\rightarrow U\cap \mathcal{M}$ such that $\phi$ is one-to-one and $\psi \circ \phi (x) = x$. 
From the continuity of ${\mathbf g}$ we get that 
${\mathbf g}^{-1}(U)\subseteq {\mathbb R}^n$
is an open neighbourhood of ${\mathbf x}^\ast$.

By construction, ${\mathbf x}^\ast\in {\mathbf g}^{-1}(U)$. Also, we have
 ${\mathbf g}({\mathbf x}) = \phi(\psi({\mathbf g}({\mathbf x})))$ for any ${\mathbf x}\in {\mathbf g}^{-1}(U)$, and, therefore, ${\rm rank\,} J_{{\mathbf g}}({\mathbf x}^\ast)  = {\rm rank\,} J_{\phi}(\psi({\mathbf g}({\mathbf x}^\ast))) J_{\psi}({\mathbf g}({\mathbf x}^\ast)) J_{\mathbf g}({\mathbf x}^\ast) \leq {\rm rank\,} J_{\phi} \leq k$.
\end{proof}

The opposite statement is not true, though we can deduce a weaker property.

\begin{theorem}\label{opposite}
Let ${\mathbf g}: S \rightarrow {\mathbb R}^n$ be an infinitely differentiable mapping on an open set $S\subseteq {\mathbb R}^n$ such that for any point ${\mathbf x}^\ast\in S$, we have ${\rm rank\,} J_{\mathbf g}({\mathbf x}^\ast) = k$. Then, the topological space $({\mathbf g}(S), T[{\mathbf g}])$, where $T[{\mathbf g}]$ is a topology generated by the base $B = \{{\mathbf g}(\Omega)\mid \Omega{\rm\,\,is\,\,open}\}$, satisfies the property that any $p\in {\mathbf g}(S)$ has a neighbourhood $A\in T[{\mathbf g}]$ that is like ${\mathbb R}^k$.
\end{theorem}
\begin{remark} 
It is natural to call $T[{\mathbf g}]$ {\bf the topology induced by ${\mathbf g}$}.
In the paper we are mainly interested in the case when $T[{\mathbf g}]$ coinsides with a subspace topology on ${\mathbf g}(S)\subseteq {\mathbb R}^n$, though the following example shows that this is not always the case. 
For $n=2, k =1$, let us define ${\mathbf g}(x,y) = {\mathbf r}(x+y)$ where ${\mathbf r}(t) = (t^2-1, t(t^2-1))$. Then, ${\mathbf g}({\mathbb R}^2) = {\mathbf r}({\mathbb R})$ is a cubic curve (which is not a 1-manifold because of the cross-like behavior around a point of self-intersection, see Figure~\ref{Serie}). 
\begin{figure}
\begin{center}
\includegraphics[width=0.4\textwidth]{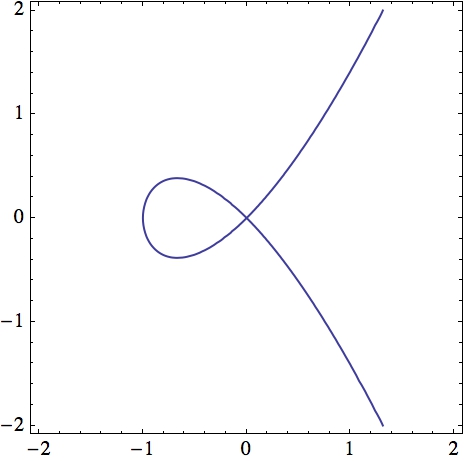}  ~~~~  
\includegraphics[width=0.25\textwidth]{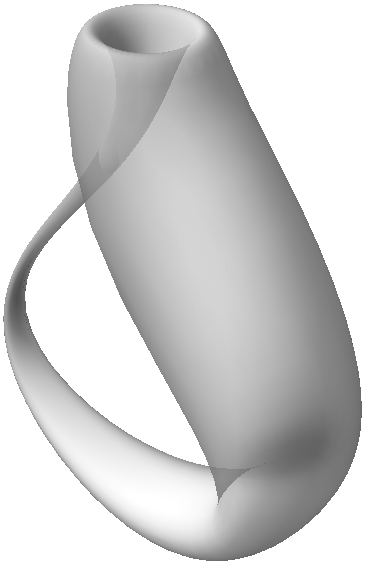}
\caption{The cubic curve (in ${\mathbb R}^2$) and the Klein bottle (in ${\mathbb R}^3$) have self-intersections.}\label{Serie}
\end{center}
\end{figure}


Though we are mostly interested in manifolds {\bf embedded} in the euclidean space (i.e. without self-intersections), autoencoders are able to train a correcting function that maps ${\mathbb R}^n$ onto a smooth manifold {\bf immersed} in ${\mathbb R}^n$ (with self-intersections). For example, if ${\mathbf g} = a\circ b$, where $b:{\mathbb R}^3\to [0,2\pi)\times [0,2\pi)$ is defined by $b(x,y,z) = \big(\frac{2\pi x^2}{x^2+1}, \frac{2\pi y^2}{y^2+1}\big)$ and $a:[0,2\pi)\times [0,2\pi)\to {\mathbb R}^3$ is some parameterization of the Klein bottle in ${\mathbb R}^3$~\cite{klein}, then $({\mathbf g}({\mathbb R}^3), T[{\mathbf g}])$ is a Klein bottle immersed in ${\mathbb R}^3$.
\end{remark}
\begin{theorem}\label{locality} Let $\mathcal{M}$ be a $k$-manifold and ${\mathbf g}: {\mathbb R}^n \rightarrow \mathcal{M}$ be infinitely differentiable.
If ${\rm rank\,} J_{\mathbf g}({\mathbf x}) = k$, then there are open $\Omega\ni {\mathbf x}, \Omega'\ni {\mathbf x}' = {\mathbf g}({\mathbf x})$ such that  $\mathcal{M}\cap \Omega' = {\mathbf g}(\Omega)$.
\end{theorem}

\begin{remark} Theorem~\ref{locality} is a practically important tool because it says that local geometry of the manifold $\mathcal{M}$ around a point ${\mathbf x}' = {\mathbf g}({\mathbf x})$ is completely defined by the behavior of ${\mathbf g}$ around ${\mathbf x}$. For example, the tangent space of $\mathcal{M}$ at point ${\mathbf x}'$ equals the column space of the Jacobian matrix, i.e. ${\rm col}(J_{\mathbf g}({\mathbf x})) = {\rm span}(\frac{\partial {\mathbf g}}{\partial x_1}, \cdots, \frac{\partial {\mathbf g}}{\partial x_n})$. Also, let us mention (without a proof) that the curvature tensor at point ${\mathbf x}'$ is completely defined by matrices $J_{\mathbf g}({\mathbf x}), H_{g_1}({\mathbf x}), \cdots, H_{g_n}({\mathbf x})$. Since this relationship is quite sophisticated, in section~\ref{Curvature} we give another second-order expression that bounds a curvature along any direction at point ${\mathbf x}'$.
\end{remark}

Let us denote ${\rm codom}({\mathbf g}) = \{{\mathbf g}({\mathbf x}) \mid  {\mathbf x}\in {\mathbb R}^n, {\rm rank\,} J_{\mathbf g}({\mathbf x}) = k\}$.
Throughout the paper let us assume that any pair $(\mathcal{M}, {\mathbf g})$ satisfies $\mathcal{M} = {\rm codom}({\mathbf g})$. This implies that for any such pair we can omit the first element of the pair, $\mathcal{M} $, because it can be recovered from ${\mathbf g}$. Thus, instead of fixing the family of $k$-manifolds $\mathfrak{M}$, we assume that we are given a set $$\mathfrak{G}\subseteq \{{\mathbf g}: {\mathbb R}^n\to {\mathbb R}^n \mid {\rm codom}({\mathbf g}){\rm \,\,is\,\, a\,\, }k{\rm -manifold}\}.$$
Theorem~\ref{primal} implies that $\mathfrak{G}\subseteq \{{\mathbf g}: {\mathbb R}^n\to {\mathbb R}^n \mid {\rm rank\,} J_{\mathbf g}({\mathbf x}) \leq k\}$.

It is now natural to re-formulate the manifold learning problem as the following minimization task:
\begin{equation}\label{jac-rank}
{\mathbf g}^\ast \leftarrow \arg\min\limits_{
{\mathbf g}\in \mathfrak{G}} \frac{1}{N} \sum_{i=1}^N \|{\mathbf x}_i-{\mathbf g}({\mathbf x}_i)\|^2.
\end{equation}
\subsection{Curvature}\label{Curvature}
One of key aspects of the formulation~\eqref{jac-rank} is that $\mathcal{M}$ is not present there directly, though after solving the task, it can be recovered by setting $\mathcal{M} = {\mathbf g}^\ast ({\mathbb R}^n)$.

Recall that the column space of $J_{{\mathbf g}}({\mathbf x}^\ast)$ is exactly the tangent space of $\mathcal{M} = {\mathbf g} ({\mathbb R}^n)$ at point ${\mathbf g}({\mathbf x}^\ast)$. Thus, the stability of $J_{{\mathbf g}}$ with respect to small perturbations of the argument ${\mathbf x}^\ast$ is equivalent to the stability of the tangent space, or, in other words, to the low curvature of the manifold $\mathcal{M}$ at point ${\mathbf g}({\mathbf x}^\ast)$. Thus, how strongly the manifold $\mathcal{M}$ deviates from the tangent hyperplane at point ${\mathbf g}({\mathbf x}^\ast)$ can be measured by the function:
$$
\kappa_1 ({\mathbf x}^\ast, \sigma) = n^{-2} {\mathbb E}_{{\boldsymbol \epsilon}\sim {\mathcal N}(0, \sigma^2 I_n)} \|J_{{\mathbf g}}({\mathbf x}^\ast+{\boldsymbol \epsilon})-J_{{\mathbf g}}({\mathbf x}^\ast)\|^2.
$$
The latter definition is motivated by the approach of~\cite{Rifai} who assumed ${\mathbf g} = {\mathbf d}\circ {\mathbf e}$, where ${\mathbf e}: {\mathbb R}^n\to {\mathbb R}^d$ is an encoder mapping and ${\mathbf d}: {\mathbb R}^d\to {\mathbb R}^n$ is a decoder mapping. In this approach the deviation from the tangent hyperplane at point ${\mathbf g}({\mathbf x}^\ast)$ was measured by the function:
$$
\kappa_0 ({\mathbf x}^\ast, \sigma) =  {\mathbb E}_{{\boldsymbol \epsilon}\sim {\mathcal N}(0, \sigma^2 I_n)} \|J_{{\mathbf e}}({\mathbf x}^\ast+{\boldsymbol \epsilon})-J_{{\mathbf e}}({\mathbf x}^\ast)\|^2.
$$

Let us introduce another way of measuring the deviation which is motivated by the classical notion of the manifold curvature.
Let $\mathcal{M}$ be a $k$-manifold and ${\mathbf x}\in \mathcal{M}$. Then, the first principal curvature of $\mathcal{M}$ at point ${\mathbf x}$ is:
\begin{equation}
\kappa ({\mathbf x}) = \sup_{\tiny \begin{matrix}
\boldsymbol{\gamma}: [-1,1]\to \mathcal{M}, \boldsymbol{\gamma}(0) = {\mathbf x},\|\boldsymbol{\gamma}'(s)\|=1, \\
\gamma {\rm \,\,is\,\, geodesic}
\end{matrix} } \|\boldsymbol{\gamma}''(0)\|,
\end{equation}
where the supremum is taken over all smooth geodesic curves on $\mathcal{M}$ that go through ${\mathbf x}$ given with the arc-length parametrization. 

We have the following upper bound on the first principal curvature.

\begin{theorem}\label{bound} Let  $\mathcal{M}\subseteq {\mathbb R}^n$ be a $k$-manifold, ${\mathbf g}: {\mathbb R}^n \rightarrow \mathcal{M}$ be a twice continuously differentiable mapping and ${\mathbf g}({\mathbf x}') = {\mathbf x}$, ${\rm rank\,}J_{{\mathbf g}}({\mathbf x}') = k$. Then,
\begin{equation}
\kappa ({\mathbf x}) \leq \overline{\lim}_{\boldsymbol{\varepsilon}\to {\mathbf 0}}\min_{\lambda}\frac{\|[\boldsymbol{\varepsilon}^T H_{g_i}({\mathbf x}') \boldsymbol{\varepsilon}]_{i=1}^n\hspace{-1pt}-\hspace{-1pt}\lambda J_{{\mathbf g}} ({\mathbf x}')\boldsymbol{\varepsilon}\|}{\|J_{{\mathbf g}}\boldsymbol{\varepsilon}\|^2}.
\end{equation}
\end{theorem}

Since ${\mathbf g} ({\mathbf x}'+\boldsymbol{\varepsilon}) - {\mathbf g}({\mathbf x}') - J_{{\mathbf g}} ({\mathbf x}')\boldsymbol{\varepsilon} \approx \frac{1}{2}[\boldsymbol{\varepsilon}^T H_{g_i}({\mathbf x}') \boldsymbol{\varepsilon}]_{i=1}^n$, the latter theorem motivates the following term to regularize the curvature:
$$
\kappa_2 ({\mathbf x}^\ast, \sigma) =  {\mathbb E}_{\boldsymbol{\varepsilon}\sim {\mathcal N}(0, \sigma^2 I_n)} \frac{2\|{\mathbf g} ({\mathbf x}^\ast\hspace{-1.5pt}+\hspace{-1.5pt}\boldsymbol{\varepsilon})\hspace{-1.5pt}-\hspace{-1.5pt}{\mathbf g}({\mathbf x}^\ast)\hspace{-1.5pt}-\hspace{-1.5pt}J_{{\mathbf g}} ({\mathbf x}^\ast)\boldsymbol{\varepsilon}\|}{\|J_{{\mathbf g}}({\mathbf x}^\ast)\boldsymbol{\varepsilon}\|^2}.
$$

One of goals of the paper is to compare these three ways to regularize a manifold's curvature, $\kappa_0$, $\kappa_1$ and $\kappa_2$. If ${\mathbf g} = {\mathbf d}\circ {\mathbf e}$, then $\kappa_0$ depends only on the encoder mapping, ${\mathbf e}$, which is simple and computationally efficient. The following theorem describes a case when the encoder itself defines the curvature to a large extent.

\begin{theorem}\label{encoder-decoder}
Let ${\mathbf g} = {\mathbf d}\circ {\mathbf e}$, where ${\mathbf e}: {\mathbb R}^n\to {\mathbb R}^k$ and ${\mathbf d}: {\mathbb R}^k\to {\mathbb R}^n$ are smooth and $\sigma_1\leq \cdots\leq \sigma_k<0$ are singular values of $J_{{\mathbf d}}({\mathbf x}')$. Then,
\begin{equation*}
\begin{split}
\overline{\lim}_{\boldsymbol{\varepsilon}\to {\mathbf 0}}\min_{\lambda}\frac{\|[\boldsymbol{\varepsilon}^T H_{g_i}({\mathbf x}') \boldsymbol{\varepsilon}]_{i=1}^n-\lambda J_{{\mathbf g}} ({\mathbf x}')\boldsymbol{\varepsilon}\|}{\|J_{{\mathbf g}}\boldsymbol{\varepsilon}\|^2}\leq 
\frac{\sigma_1 C({\mathbf e})}{\sigma^2_k} + \frac{\sqrt{\sum_{i=1}^n \|H_{d_i}\|_F^2 }}{\sigma^2_k},
\end{split}
\end{equation*}
where $C({\mathbf e}) = \overline{\lim}_{\boldsymbol{\varepsilon}\to {\mathbf 0}}\min_{\lambda} \frac{\|[\boldsymbol{\varepsilon}^T H_{e_i}({\mathbf x}') \boldsymbol{\varepsilon}]_{i=1}^n-\lambda J_{{\mathbf e}} ({\mathbf x}')\boldsymbol{\varepsilon}\|}{\|J_{{\mathbf e}}\boldsymbol{\varepsilon}\|^2}$.
\end{theorem}
This theorem shows that when the decoder has a well-defined Jacobian matrix and a small Hessian matrix, i.e. $\frac{\sigma_1}{\sigma^2_k}$ is moderate and $\|H_{d_i}\|_F/\sigma^2_k\approx 0$, then the curvature can be bounded by $\mathcal{O}\big(C({\mathbf e})\big)$ which depends only on the encoder mapping. 

Our experiments show that when the decoder is non-linear (a 2-layer neural network), then regularization terms $\kappa_1, \kappa_2$ are preferable to $\kappa_0$. We explain this observation by the fact that during the training process the norm of decoder's Hessian can be comparable to $\sigma_k^2$. Thus, one cannot neglect decoder's effect on the curvature of a manifold.

\subsection{Two approaches to force rank constraints}\label{softly}
{\bf The first approach: the autoencoder architecture.} We search for the correcting function in the form of an autoencoder, i.e. 
$$
{\mathbf g}_\theta({\mathbf x}) = {\mathbf d}_\theta({\mathbf e}_\theta({\mathbf x})),
$$
where ${\mathbf e}_\theta: {\mathbb R}^n\rightarrow {\mathbb R}^d$ and ${\mathbf d}_\theta: {\mathbb R}^d\rightarrow {\mathbb R}^n$ are feedforward neural networks and $d\geq k$. Note that this architecture satisfies the condition $ \forall {\mathbf x}: {\rm rank\,} J_{{\mathbf g}_\theta}({\mathbf x}) \leq d$ by construction. 

{\bf The second approach: soft forcing.} In the optimization task~\eqref{jac-rank} we have to search over a space of functions ${\mathbf g}: {\mathbb R}^n\rightarrow {\mathbb R}^n$ whose Jacobian satisfies ${\rm rank\,} J_{{\mathbf g}}({\mathbf x}) \leq k$. 

\begin{definition} 
Let $\sigma_1 (A)\geq \sigma_2 (A)\geq \cdots \geq \sigma_{\min \{a,b\}} (A)$ be ordered singular values of the matrix $A\in {\mathbb R}^{a\times b}$ (counting multiplicities). Then, $\|A\|^{(k)} \eqdef \sqrt{\sum_{i=k+1}^{\min \{a,b\}}\sigma^2_i (A)}$ and is called the Ky-Fan $k$-antinorm. Also, let us denote $\|A\|_k \eqdef \big(\|A\|^{(k)}\big)^2$.
\end{definition}

To relax the constraint ${\rm rank\,}J_{\mathbf g}\leq k$ we substitute it with an additional penalty term $\lambda T({\mathbf g})$, in which 
$$
T({\mathbf g}) = \frac{1}{M}\sum_{j=1}^{M} \|J_{{\mathbf g}}({\mathbf x}_{i_j})\|_k.
$$
It is easy to see that $\|A\|_k \geq 0$ and $\|A\|_k = 0$ if and only if ${\rm rank}\, A = k$. Thus, the term $\lambda T({\mathbf g})$ forces the Jacobian condition:  ${\rm rank\,} J_{{\mathbf g}}({\mathbf x}_{i_j}) \leq k$. 
A set $\{{\mathbf x}_{i_1}, \cdots, {\mathbf x}_{i_M}\}$ is a subset of data points at which we force the Jacobian to be of rank not greater than $k$. Below we describe an algorithm that requires ${\mathcal O}(M n^3)$ time and ${\mathcal O}(M n^2)$ space which is quite limiting for $n\sim 10^3$. Therefore, in the objective we sum $\|J_{\mathbf g} ({\mathbf x}_{i})\|_k$ over a certain subset of $\{{\mathbf x}_{1}, \cdots, {\mathbf x}_{N}\}$ of size $M \ll N$, rather than over the whole dataset.

Our final penalty formulation is equivalent to the minimization of the following objective:
\begin{equation}\label{penalty}
\begin{split}
\Phi({\mathbf g}) = \frac{1}{N}\sum_{i=1}^N \big\{\|{\mathbf x}_i-{\mathbf g}({\mathbf x}_i)\|^2 + 
\gamma  \kappa_o({\mathbf x}_i)\big\} + 
\frac{\lambda}{M} \sum_{j=1}^{M} \|J_{{\mathbf g}}({\mathbf x}_{i_j})\|_k ,
\end{split}
\end{equation}
where $o=0,1$ or $2$, $\gamma$ is a curvature regularization parameter and $\lambda$ is a soft rank regularization parameter.

{\bf Joining two approaches.} Thus, the parameter $k$ should be chosen to be smaller than $d$, and the term $\|J_{{\mathbf g}_\theta}({\mathbf x}_{i_j})\|_k $ can be considered as the term that reduces of the dimension of the manifold additional to the reduction that is defined by the autoencoder's architecture.
In other words, we search for ${\mathbf f} ={\mathbf g}_{\theta^\ast} $ where
\begin{equation}
\theta^\ast \leftarrow \arg\min\limits_{\theta\in \Theta} \Phi ({\mathbf g}_\theta).
\end{equation}

\section{The alternating algorithm}\label{a-algoritm}
The Eckart-Young-Mirsky theorem gives
$$
\|A\|_k = \min_{B: {\rm rank\,}B\leq k}\|A-B\|_F^2,
$$
where $\|\cdot\|_F$ is the Frobenius distance. Thus, our goal is to minimize the objective
\begin{equation}\label{penalty}
\begin{split}
F(\theta, \langle B_j\rangle_{j=1}^M ) =  \sum_{i=1}^N \big\{\|{\mathbf x}_i-{\mathbf g}_\theta({\mathbf x}_i)\|^2 + 
\gamma  \kappa_o({\mathbf x}_i)\big\} + 
\lambda \sum_{j=1}^{M} \|J_{{\mathbf g}_\theta}({\mathbf x}_{i_j})-B_{{\mathbf x}_{i_j}}\|^2_F 
\end{split}
\end{equation}
simultaneously over the parameter $\theta$ and over matrices $B_{{\mathbf x}_{i_j}}$ such that ${\rm rank\,}B_{{\mathbf x}_{i_j}}\leq k$, $j=\overline{1,M}$.
Note that for fixed $\theta$, the minimization over $\{B_{{\mathbf x}_{i_j}}\}_{j=1}^M$ is equivalent to setting 
$$
B_{{\mathbf x}_{i_j}} \leftarrow U^j_{1:n, 1:k}\Sigma^j_{1:k, 1:k} (V^j_{1:n, 1:k})^T,
$$
where $J_{{\mathbf g}_\theta}({\mathbf x}_{i_j}) = U^j\Sigma^j (V^j)^T$ is the singular value decomposition of $J_{{\mathbf g}_\theta}({\mathbf x}_{i_j})$ (this also follows from the Eckart-Young-Mirsky theorem). For fixed matrices $B_{{\mathbf x}_{i_j}}$, $j=\overline{1,M}$, we can minimize over $\theta$ using any standard gradient technique, e.g. the Adam optimizer. Thus, it is natural to minimize our function in an alternating fashion, i.e. first optimize over $\theta$, then over $B_{{\mathbf x}_{i_j}}$, $j=\overline{1,M}$, then over $\theta$  etc., until our algorithm converges. We call this approach {\em the alternating algorithm} and its pseudocode~\ref{alternate} is given below. 


\begin{algorithm}
\begin{algorithmic}\scriptsize
\caption{The alternating algorithm. Hyperparameters: $m,\lambda, \gamma, \sigma, \alpha, \beta_1, \beta_2$}\label{alternate}
\FOR{$j = 1, \cdots, M$} 
\STATE $B_{{\mathbf x}_{i_j}} \longleftarrow {\mathbf 0}$
\ENDFOR
\FOR{$t = 1, \cdots, T$}

\WHILE{$\theta$ has not converged}

\STATE Sample $\{{\mathbf y}_i\}_{i=1}^m \sim P_{{\rm data}({\mathbf x}_{1}, \cdots, {\mathbf x}_{N})}$
\STATE Sample $\{\text{\boldmath$\epsilon$}_i\}_{i=1}^m \sim {\mathcal N}(0, \sigma^2)$
\STATE Sample $\{{\mathbf z}_i\}_{i=1}^m \sim P_{{\rm data}({\mathbf x}_{i_1}, \cdots, {\mathbf x}_{i_M})}$
\STATE $L\longleftarrow \frac{1}{m}\sum_{i=1}^m ({\mathbf y_i}- {\mathbf g}_\theta ({\mathbf y_i}))^2 + \frac{\gamma}{m}\sum_{i=1}^{m} \|J_{{\mathbf g}_\theta} ({\mathbf y_i}+\text{\boldmath$\epsilon$}_i) - J_{{\mathbf g}_\theta} ({\mathbf y_i})\|^2 + \frac{\lambda}{m}\sum_{i=1}^{m} \|J_{{\mathbf g}_\theta} ({\mathbf z_i}) - B_{{\mathbf z_i}}\|^2$ (the case of $\kappa_1$)
\STATE $L\longleftarrow \frac{1}{m}\sum_{i=1}^m ({\mathbf y_i}- {\mathbf g}_\theta ({\mathbf y_i}))^2 + \frac{\gamma}{m}\sum_{i=1}^{m} \frac{\|{\mathbf g}_\theta ({\mathbf y_i}+\text{\boldmath$\epsilon$}_i) - {\mathbf g}_\theta ({\mathbf y_i}) - J_{{\mathbf g}_\theta} ({\mathbf y_i})\text{\boldmath$\epsilon$}_i\|^2}{\|J_{{\mathbf g}_\theta} ({\mathbf y_i})\text{\boldmath$\epsilon$}_i\|^2} + \frac{\lambda}{m}\sum_{i=1}^{m} \|J_{{\mathbf g}_\theta} ({\mathbf z_i}) - B_{{\mathbf z_i}}\|^2$ (the case of $\kappa_2$)
\STATE $\theta \longleftarrow {\rm Adam} (\nabla_{\theta} L, \theta, \alpha, \beta_1, \beta_2)$

\ENDWHILE

\STATE $\theta_t \longleftarrow \theta$

\FOR{$j = 1, \cdots, M$} 
\STATE $U^j\Sigma^j(V^j)^T \longleftarrow {\rm SVD}(J_{{\mathbf g}_{\theta_t}}({\mathbf x}_{i_j}))$
\STATE $B_{{\mathbf x}_{i_j}} \longleftarrow U^j_{1:n, 1:k}\Sigma^j_{1:k, 1:k} (V^j_{1:n, 1:k})^T$
\ENDFOR
\ENDFOR
\STATE \textbf{Output:} ${\mathbf f}\longleftarrow {\mathbf g}_{\theta_T}$
\end{algorithmic}
\end{algorithm}

\begin{remark} An analysis of the computational complexity of the algorithm~\ref{alternate} is an open problem even for the simplest architecture of an autoencoder ${\mathbf g}_{W}({\mathbf x})=W W^T {\mathbf x}, W\in {\mathbb R}^{n\times d}$. In all our experiments $T=1000$ was enough for the algorithm to converge. The computation of ${\rm SVD}(J_{{\mathbf g}_{\theta_t}}({\mathbf x}_{i_j}))$ is a time-consuming part of the algorithm. Since ${\rm SVD}$ requires a cubic time to compute, each iteration of our algorithm requires at least $\mathcal{O}(Mn^3)$ flops. In our implementation we exploited the fact $J_{{\mathbf d}\circ  {\mathbf e}} ({\mathbf x})= J_{{\mathbf d}}({\mathbf e}({\mathbf x})) J_{{\mathbf e}}({\mathbf x})$ and a standard reduction of the SVD of the product to a combination of QR decomposition and SVD of multiples~\cite{SVD_of_product}. This substantially accelerates the algorithm when $d\ll n$, though poor scalability for $d>1000, n>3000$ maintains.
\end{remark}
\section{Experiments}\label{experiments}
We conduct the experiments with the algorithm~\ref{alternate} that we apply to unsupervised synthetic data and the following datasets: Higgs~\cite{higgs} and Forest Cover Type~\cite{forest}, MNIST~\cite{mnist}, CIFAR-10~\cite{krizhevsky2009learning}, SVHN~\cite{YuvalNetzer}, STL-10~\cite{pmlr-v15-coates11a}.

\subsection{Experiments on synthetic data}
We select natural numbers $n_2 < n_1$ and define $\mathcal{M} = {\rm St}(n_1, n_2)$ where ${\rm St}(n_1, n_2)$ is a Stiefel manifold, i.e. $\mathcal{M} = \{X\in {\mathbb R}^{n_1\times n_2}| X^TX=I_{n_2}\}$. By reshaping an $n_1\times n_2$-matrix to $n_1n_2$ dimensional vector one can embed $\mathcal{M}\hookrightarrow {\mathbb R}^{n_1n_2}$, making $\mathcal{M}$ a $n_1 n_2 - \frac{1}{2}n_2(n_2+1)$-dimensional manifold in ${\mathbb R}^{n_1n_2}$. We sample a set of vectors ${\mathbf z}_1, \cdots, {\mathbf z}_N\in {\mathbb R}^{n_1n_2}$ uniformly on $\mathcal{M}$, generate noise $\boldsymbol{\epsilon}_i\sim^{\rm iid} \mathcal{N}({\mathbf 0},\delta^2I_{n_1n_2})$, and set ${\mathbf x}_i = {\mathbf z}_i+\boldsymbol{\epsilon}_i$. CAE with a Hessian method~\cite{Rifai} and the alternating algorithm (AS)~\ref{alternate} are applied to the dataset $\{{\mathbf x}_i \}_{i=1}^N$. A result of a manifold learning algorithm is an autoencoder $a: {\mathbb R}^{n_1n_2}\to {\mathbb R}^{n_1n_2}$. We use two criteria to assess the quality of the manifold $a({\mathbb R}^{n_1n_2})$. The first,
\begin{equation*}
\begin{split}
e_1 = {\mathbb E}_{{\mathbf z}\sim {\rm Uniform}(\mathcal{M}), \boldsymbol{\epsilon}\sim \mathcal{N}({\mathbf 0},\delta^2I_{n_1n_2})} \big\{ \|[a({\mathbf z}+\boldsymbol{\epsilon})]_{n_1\times n_2}^T [a({\mathbf z}+\boldsymbol{\epsilon})]_{n_1\times n_2} - I_{n_2}\|_{\infty}\big\},
\end{split}
\end{equation*}
measures how close to an orthogonal matrix is an output of $a$, reshaped to an $n_1\times n_2$-matrix. The second,
$$
e_2 = \frac{{\mathbb E}_{{\mathbf z}\sim {\rm Uniform}(\mathcal{M}), \boldsymbol{\epsilon}\sim \mathcal{N}({\mathbf 0},\delta^2I_{n_1n_2})} \|a({\mathbf z}+\boldsymbol{\epsilon})-{\mathbf z}\|}{\sqrt{n_1n_2\delta^2}},
$$
measures how accurately $a$ recovers ${\mathbf z}\sim {\rm Uniform}(\mathcal{M})$.

We used CAE+H method with an autoencoder architecture 
\begin{equation*}
\begin{split}
{\mathbb R}^{n_1n_2}\overset{\sigma(W_1{\mathbf x}+{\mathbf b}_1)}{\rightarrow}{\mathbb R}^{8n_1n_2}\overset{\sigma(W_2{\mathbf x}+{\mathbf b}_2)}{\rightarrow}\\
{\mathbb R}^{n_1n_2-\frac{1}{2}n_2(n_2+1)}\overset{\sigma(W_3{\mathbf x}+{\mathbf b}_3)}{\rightarrow}{\mathbb R}^{8n_1n_2}\overset{W_4{\mathbf x} +{\mathbf b}_4}{\rightarrow}{\mathbb R}^{n_1n_2},
\end{split}
\end{equation*}
which guarantees (by construction) that ${\rm rank\,}J_{{\mathbf g}_\theta}\leq n_1n_2-\frac{1}{2}n_2(n_2+1)$.
The alternating algorithm (AS) was applied with an autoencoder architecture 
\begin{equation*}
\begin{split}
{\mathbb R}^{n_1n_2}\overset{\sigma(W_1{\mathbf x}+{\mathbf b}_1)}{\rightarrow}{\mathbb R}^{8n_1n_2}\overset{\sigma(W_2{\mathbf x}+{\mathbf b}_2)}{\rightarrow}\\
{\mathbb R}^{8n_1n_2}\overset{\sigma(W_3{\mathbf x}+{\mathbf b}_3)}{\rightarrow}{\mathbb R}^{8n_1n_2}\overset{W_4{\mathbf x} +{\mathbf b}_4}{\rightarrow}{\mathbb R}^{n_1n_2}.
\end{split}
\end{equation*} In the AS method we forced the Jacobian rank constraints by setting $k=n_1 n_2 - \frac{1}{2}n_2(n_2+1)$.
Note that we selected wide intermediate layers (8 times the input size) because the projection operator onto the Stiefel manifold appears to be uneasy to approximate by a narrow 4-Layer neural network. 
    \begin{figure*}
        \centering
        \begin{tabular}{ccc}
           
             \includegraphics[scale = 0.2]{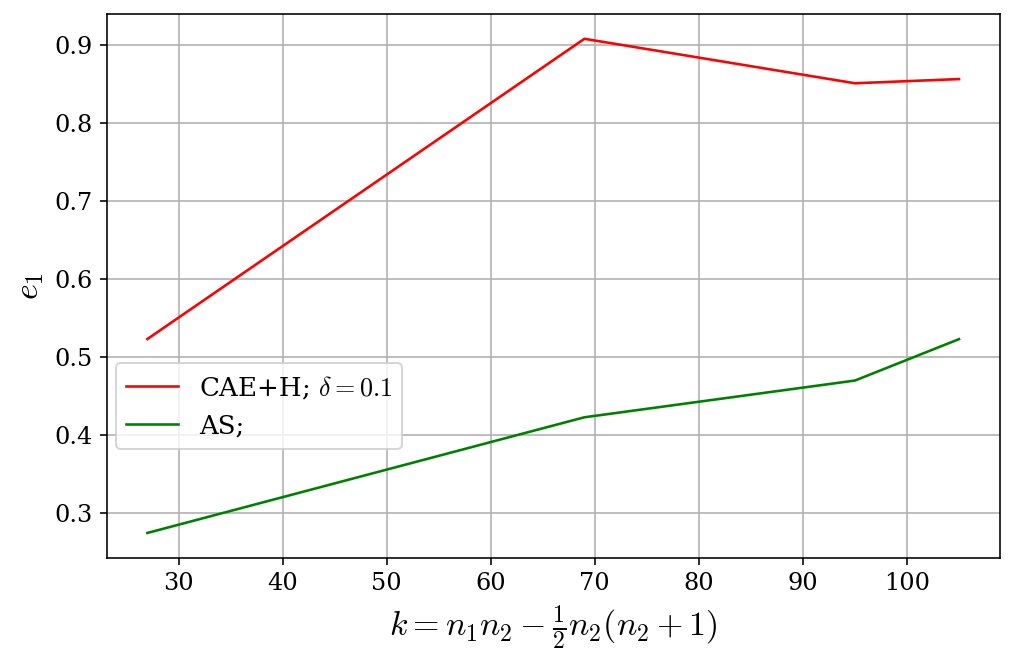} ~  &  \includegraphics[scale = 0.2]{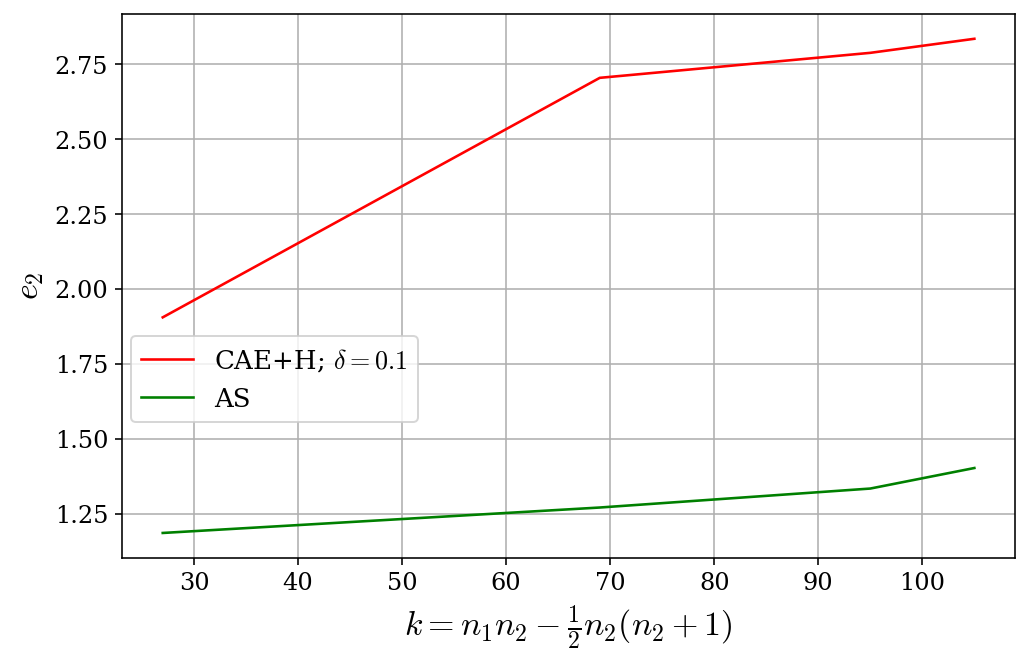} ~  & \includegraphics[scale = 0.2]{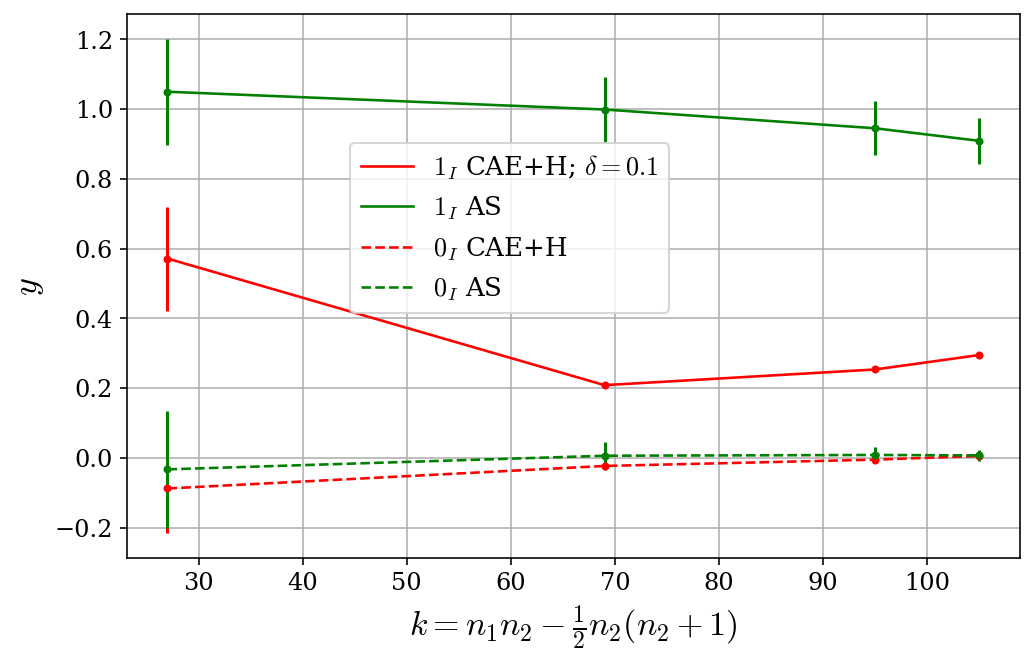} \\
             \includegraphics[scale = 0.2]{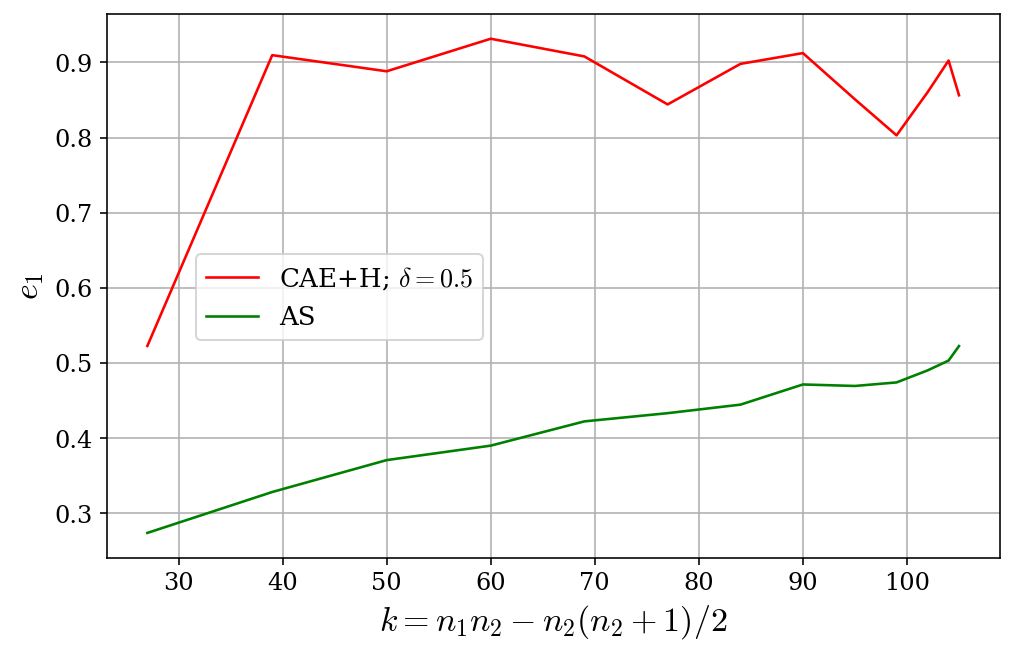} ~  & \includegraphics[scale = 0.2]{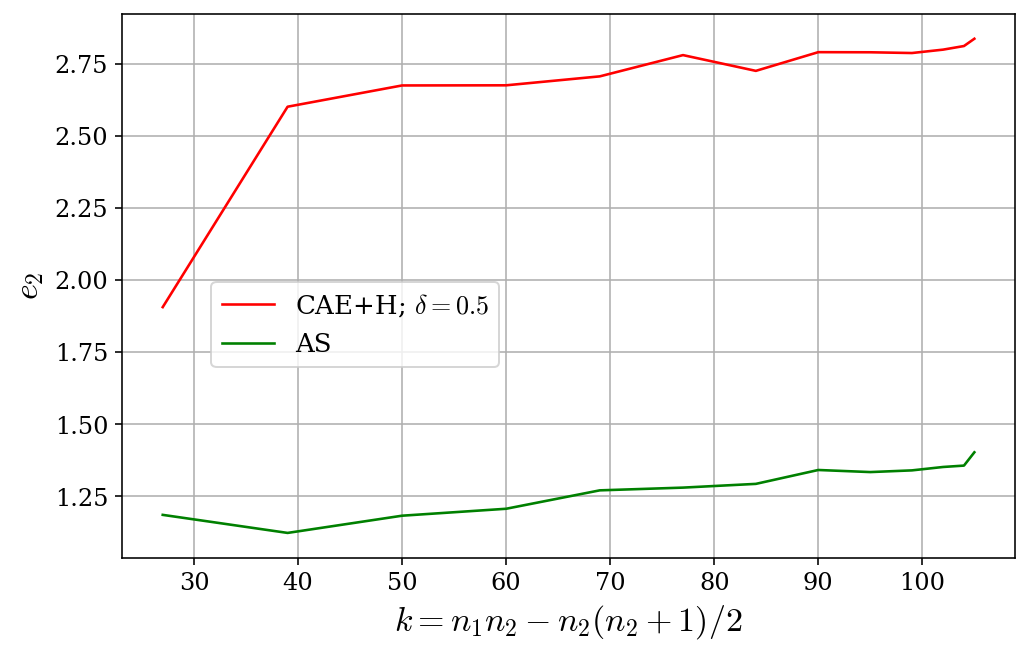} ~  & \includegraphics[scale = 0.2]{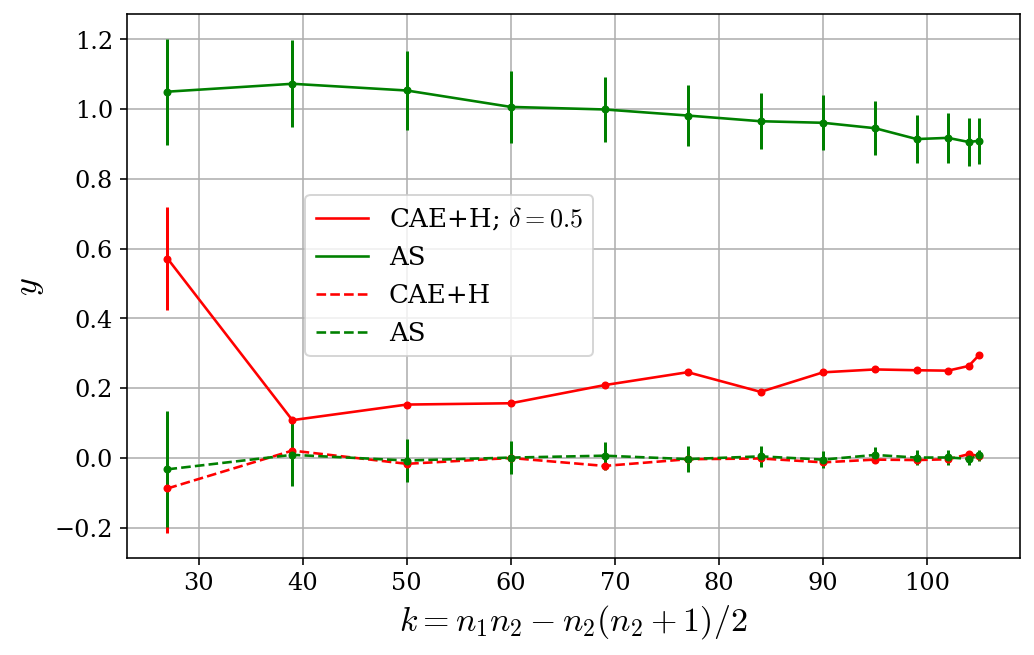}
        \end{tabular}
        \caption{$e_1$, $e_2$, $1_I$ and $0_I$ as a function of the hidden dimensionality $k$, for CAE+H and AS.}
        \label{tbl:table_of_figures}
    \end{figure*}

{\bf Discussion}. A dependence of $e_1$ and $e_2$ on the ``hidden'' dimensionality $k=n_1n_2-\frac{n_2(n_2+1)}{2}$, for fixed $n_1=15$, is given on Figure~\ref{tbl:table_of_figures}. Both errors monotonically grow with $k$, and the alternating algorithm performs better than CAE+H. The third picture shows average values of diagonal elements $1_I = \frac{1}{n^2}\sum_{i,k} a({\mathbf z}+\boldsymbol{\epsilon})_{ki} a({\mathbf z}+\boldsymbol{\epsilon})_{ki}$ and off-diagonal elements $0_I = \frac{1}{n^2(n-1)}\sum_{i\ne j}\sum_k a({\mathbf z}+\boldsymbol{\epsilon})_{ki} a({\mathbf z}+\boldsymbol{\epsilon})_{kj}$, $n=n_1n_2$, on the test set. Ideally, $1_I\approx 1$ and $0_I\approx 0$, which is better satisfied by the autoencoder trained by AS. Note that the autoencoder trained by CAE+H also has $0_I\approx 0$, though diagonal elements are systematically smaller than 1. We conjecture that this phenomenon is due to the fact that the Stiefel manifold is not diffeomorphic to ${\mathbb R}^k$ (in general, it is not connected and not simply connected), i.e. architectures with a $k$-dimensional euclidean code space are not able to learn such a manifold. On the contrary, the architecture of the AS allows to define the code space with a larger dimensionality ($d=8n_1n_2$ in our case), and to force a $k$-dimensionality of the trained manifold by the Ky-Fan antinorm term.
\subsection{Manifold tangent classifier and $K$-NN}
 Let $\{({\mathbf x}_i, y_i)\}_{i=1}^N, {\mathbf x}_i\in {\mathbb R}^n, y_i\in \mathcal{C}$ be a dataset of a classification task with a finite number of classes $\mathcal{C}$.
Let $\theta^\ast \in \Theta$ be the value of the autoencoder's parameter that is pretrained on the unsupervised dataset 
$\{{\mathbf x}_i\}_{i=1}^N$. Thus, our autoencoder is given as ${\mathbf g}_{\theta^\ast}({\mathbf x}) = {\mathbf d}_{\theta^\ast}({\mathbf e}_{\theta^\ast}({\mathbf x}))$.
We use two basic methods to estimate the quality of the manifold recovered from data (or, alternatively, of the autoencoder ${\mathbf g}_{\theta^\ast}$): $K$ nearest neighbors algorithm ($K$-NN) and the manifold tangent classifier method (MTC).
In the $K$-NN we classify any ${\mathbf x}\in {\mathbb R}^n$ using $K$ nearest neighbours algorithm applied to the code space representations, i.e. $\{{\mathbf e}_{\theta^\ast}({\mathbf x}_i)\}_{i=1}^N$ serves as the training set and we look for $K$ nearest neighbours of ${\mathbf e}_{\theta^\ast}({\mathbf x})$. The accuracy attained on the test set shows the quality of the encoder mapping ${\mathbf e}_{\theta^\ast}$.

In the MTC method, we search for a classifier in the form:
$$
p_{\scaleto{\{{\mathbf w}_c\}_{c\in \mathcal{C}}, \theta}{5pt}}(y=c| {\mathbf x}) \propto e^{{\mathbf w}_c\cdot {\mathbf e}_{\theta}({\mathbf x})}
$$
for each $c\in \mathcal{C}$. Training of the classification model (over parameters $\{{\mathbf w}_c\}_{c\in \mathcal{C}}, \theta$) is made by the Adam optimizer with the cross-entropy loss regularized by the tangent propagation terms:
$$
\Omega({\mathbf x}) = \beta \sum_{i=1}^k \|J_{{\mathbf w}_c\cdot {\mathbf e}_{\theta}}({\mathbf x}) {\mathbf u}_i\|^2,
$$
where $\{{\mathbf u}_i\}_{i=1}^k$ is basis of the tangent space of ${\mathbf g}_{\theta^\ast}({\mathbb R}^n)$ at point ${\mathbf g}_{\theta^\ast}({\mathbf x})$. This basis is calculated as the first $k$ left singular vectors of $J_{{\mathbf g}_{\theta^\ast}}({\mathbf x})$ (or, the first $k$ right singular vectors of $J_{{\mathbf e}_{\theta^\ast}}({\mathbf x})$ in MTC of~\cite{Rifai2}).
The key aspect of the MTC method is the initialization $\theta^0 = \theta^\ast$ of the optimizer that makes deeper levels of the neural network to be close to the pre-trained encoder function, ${\mathbf w}^0_c$ is initialized randomly.

\subsection{Experiments on real world data}\label{rwd}
Following~\cite{Rifai}, after training an autoencoder, we use the reduced encoding of an input point as a new object representation. Given that representation, we subsequently exploit it for the classification of test set data with the $K$-NN algorithm for $K=1, 2, \cdots, 19$ and the MTC method. 
We use 4-Layer autoencoder with dimensions of hidden layers $n_1, n_2, n_1$, i.e. our encoder is a composition of the form ${\mathbb R}^n\to {\mathbb R}^{n_1}\to {\mathbb R}^{n_2}$ and a decoder is a composition of the form ${\mathbb R}^{n_2}\to {\mathbb R}^{n_1}\to {\mathbb R}^{n}$. 
Parameters of an autoencoder training procedure (code space sizes $n_1, n_2$, the dimension $k$, etc) and of the MTC method ($\beta$) were tuned on a validation set. See details of implementation in Appendix. An accuracy of the classification on a test set attests to the quality of obtained representation, or the usefulness of our manifold learning technique. Alternative representations are obtained using the CAE+Hessian method~\cite{Rifai}. Details can be found on \href{https://github.com/cvfrs/as}{github}. Results are reported in the table~\ref{results} (AS denotes the algorithm~\ref{alternate}).
\begin{table*}
\scriptsize
\centering
\begin{tabular}{|c |c | c | c | c | c | c| c| c|} 
\hline
 Data & \multicolumn{2}{c|}{CAE+H} &   \multicolumn{2}{c|}{AS ($\kappa_1$)}    &  \multicolumn{2}{c|}{AS ($\kappa_2$)}  & \multicolumn{2}{c|}{CAE+H (6L)}\\ [0.4ex] 
 \hline
 set &  $K$-NN & MTC &   $K$-NN  & MTC   & $K$-NN  & MTC  & $K$-NN & MTC \\ [0.4ex] 
\hline
Higgs & 84.6 &  83.3 & {\bf 85.5}  & {\bf 83.4} &  85.0   & {\bf 83.4} & 84.8 & 83.2\\
\hline

Forest &  86.6 & 84.2 &  {\bf 88.5} & {\bf 86.3} &  87.6   & 85.8 & 86.3 & 84.8\\
\hline

MNIST &  95.9 & 96.2 & {\bf 98.6} &  {\bf 98.2} &  98.2  & {\bf 98.2} & 97.2 & 97.0\\
\hline

CIFAR-10 &  46.7 & 47.7 &  48.3 & 47.9  &  {\bf 48.4}  & {\bf 48.0} & 40.9 & 41.5\\
\hline

SVHN &  56.0  &  60.6  &  {\bf 70.5} &  {\bf 73.6} &  70.4 & {\bf 73.6}  & 59.2 & 60.6\\
\hline

STL-10 &  37.2  &  33.3  &  37.0  & {\bf 37.3}  & {\bf 37.5}  & {\bf 37.3}  & 36.0 & 34.5 \\
\hline

\end{tabular}
\caption{Results of experiments for a 4-Layers autoencoder, with a 2-Layers encoder (dimensions of layers are $n_1, n_2$) and a 2-Layers decoder (dimensions of layers are $n_2, n_1$). Last two columns list results for a 6-Layers autoencoder, with a 3-Layers encoder (dimensions of  layers are $n, n_1, n_2, k$) and a 3-Layers decoder (dimensions of layers are $k, n_2, n_1, n$).}\label{results}
\end{table*}

{\bf Discussion}.
An autoencoder trained by the AS with a curvature-related term $\kappa_1$ demonstrates the best accuracy of the MTC method on a test set. We conclude that the AS with the $\kappa_1$-term calculates tangent spaces more accurately than the AS with the $\kappa_2$-term, and the latter method is more accurate than the CAE+H with the $\kappa_0$-term. 

The MTC method uses bases of tangent spaces, computed from a pre-trained autoencoder. For all datasets we observed that an optimal value of $\beta$ is either 0.01 or 0.0, which indicates that the role of the tangent propagation term is less important than the role of the initialization. Recall that in the MTC method we initialize the first two layers of a neural network for classification by pre-trained weights of the encoder.

Sometimes the accuracy of the KNN is higher than that of the MTC, which is 
another interesting aspect of our experimental results. 
A possible explanation for this fact one can imagine could be that two hidden layers of an encoder are already optimally tuned on unsupervised data, and a supervised learning part of the MTC only tunes the final layer while leaving the previous two layers unchanged. If this is the case, then the fitting capacity of the MTC can be even weaker than that of the KNN.

The last two columns of Table~\ref{results} demonstrate the results of experiments with CAE+H ($K$-NN and MTC) and a 6-Layers autoencoder in which an encoder is a composition of the form ${\mathbb R}^n\to {\mathbb R}^{n_1}\to {\mathbb R}^{n_2}\to {\mathbb R}^{k}$ and a decoder is a composition of the form ${\mathbb R}^{k}\to {\mathbb R}^{n_2}\to {\mathbb R}^{n_1}\to {\mathbb R}^{n}$ (code sizes $n_1, n_2$ are the same as in the previous experiments with 4-Layers autoencoder). Note that the rank of such an autoencoder does not exceed $k$ by construction. The motivation of these experiments was to compare obtained accuracies with the AS method of training an autoencoder in which a rank is reduced by the soft term, i.e. the Ky Fan $k$-antinorm term. Our results demonstrate that indeed the soft way of reducing a rank surpasses the architecture-based ``crisp'' reduction.

Finally, we experimented with a joint influence of parameters $\lambda$ and $\gamma$ on the accuracy of the 1-NN classifier (and, therefore, on the quality of representation of data points in the latent space). One can see level curves of the accuracy as a function of $\log_2\lambda$ and $\gamma$ on Figure~\ref{contour}. Since the Algorithm~\ref{alternate} does not directly optimize the accuracy of 1-NN, the latter value does not depend on $\lambda, \gamma$ in a simple deterministic way. So, level curves demonstrate some local minima and maxima which can be explained by the random nature of the experiment.  The accuracy of 1-NN for the Forest dataset varies between 76.0\% and 86.5\% throughout the whole domain, which indicates that the recovered manifold is robust with respect to changes in $\lambda, \gamma$. Also, Figure~\ref{contour} clearly demonstrates one global maximum located at $(\log_2\lambda, \gamma)\approx (5.5, 4)$.
If the existence of an optimal $\gamma$ is natural (because $\gamma$ controls how strongly a manifold can be bent), the existence of an optimal $\lambda$ should be explained. Indeed, $\lambda\to+\infty$ is just equivalent to requiring a rank of the Jacobian to be strictly less or equal to $k$ at specified points. We believe that the decay of the accuracy for $\lambda\to+\infty$ occurs due to the non-convexity of the Ky Fan $k$-antinorm. In fact, it is well-known that the Ky Fan $k$-antinorm has exponentially many local minima. As $\lambda\to+\infty$, these local minima start to dominate and the minimization of the objective~\eqref{penalty} can lead us to a landscape of low-quality solutions. An analogous picture is observed for the SVHN dataset.

Contour plots for the accuracy of the MTC are clearly more random, which can be seen from a smaller variation of the accuracy (between 82.2\% and 83.7\% for the Forest dataset) and a lack of a noticeable global maximum. This picture can be explained by the fact that the MTC further optimizes an encoder's weights and eventually converges to its own optimal point. Thus, the quality of the MTC is a more sophisticated function of parameters $\lambda,\gamma$.
    \begin{figure*}
        \centering
        \begin{tabular}{ccc}
             \includegraphics[scale = 0.45]{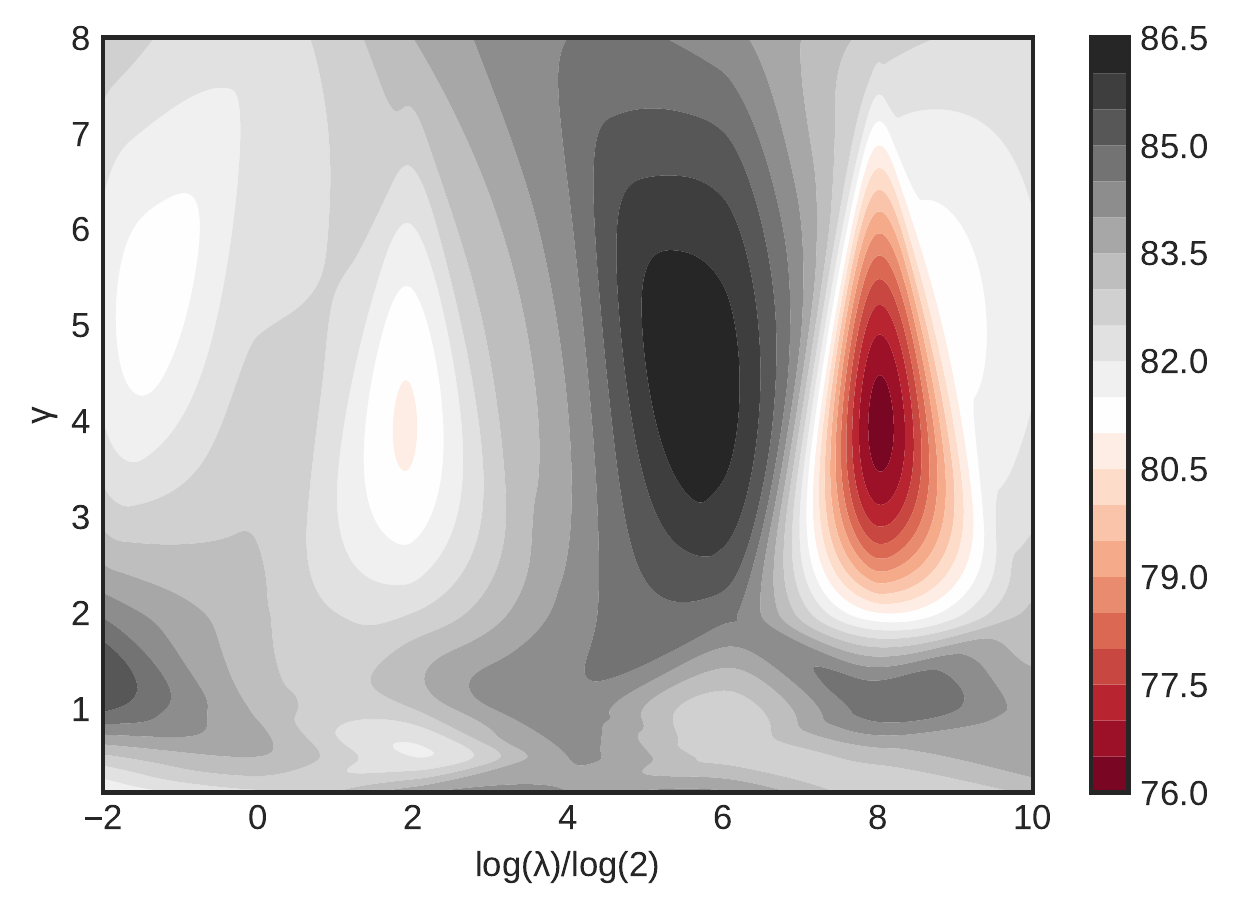} ~  &  \includegraphics[scale = 0.45]{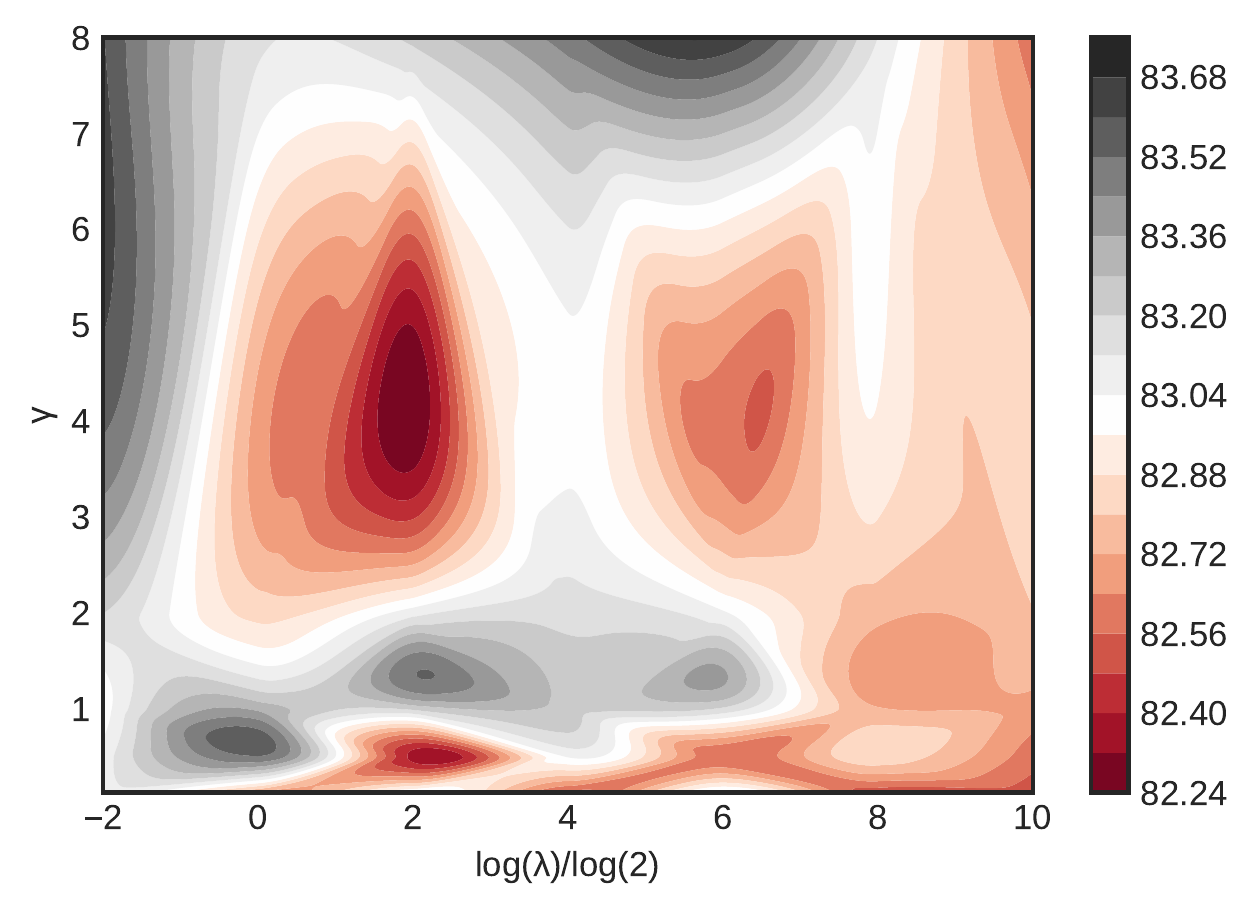} \\
             \includegraphics[scale = 0.45]{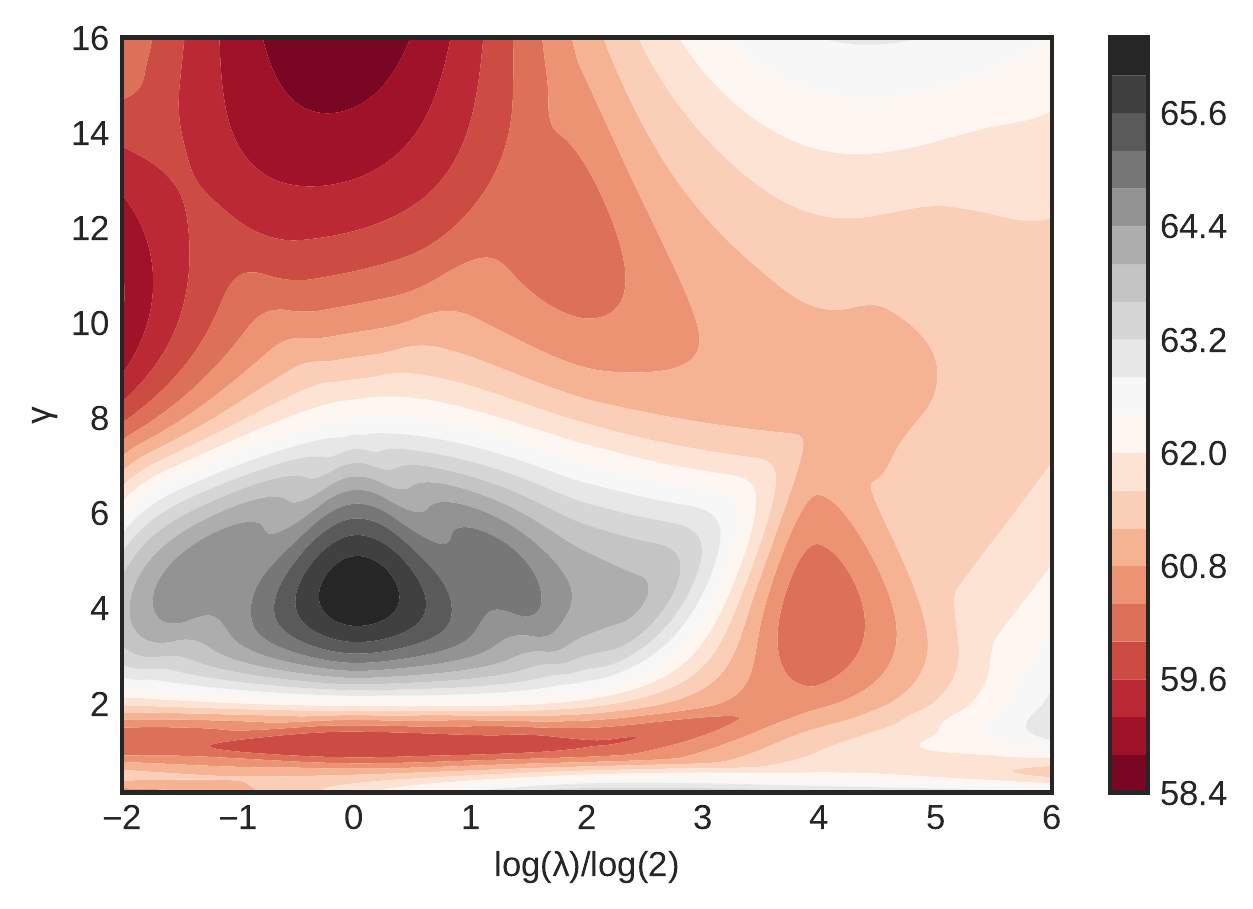} ~  &  \includegraphics[scale = 0.45]{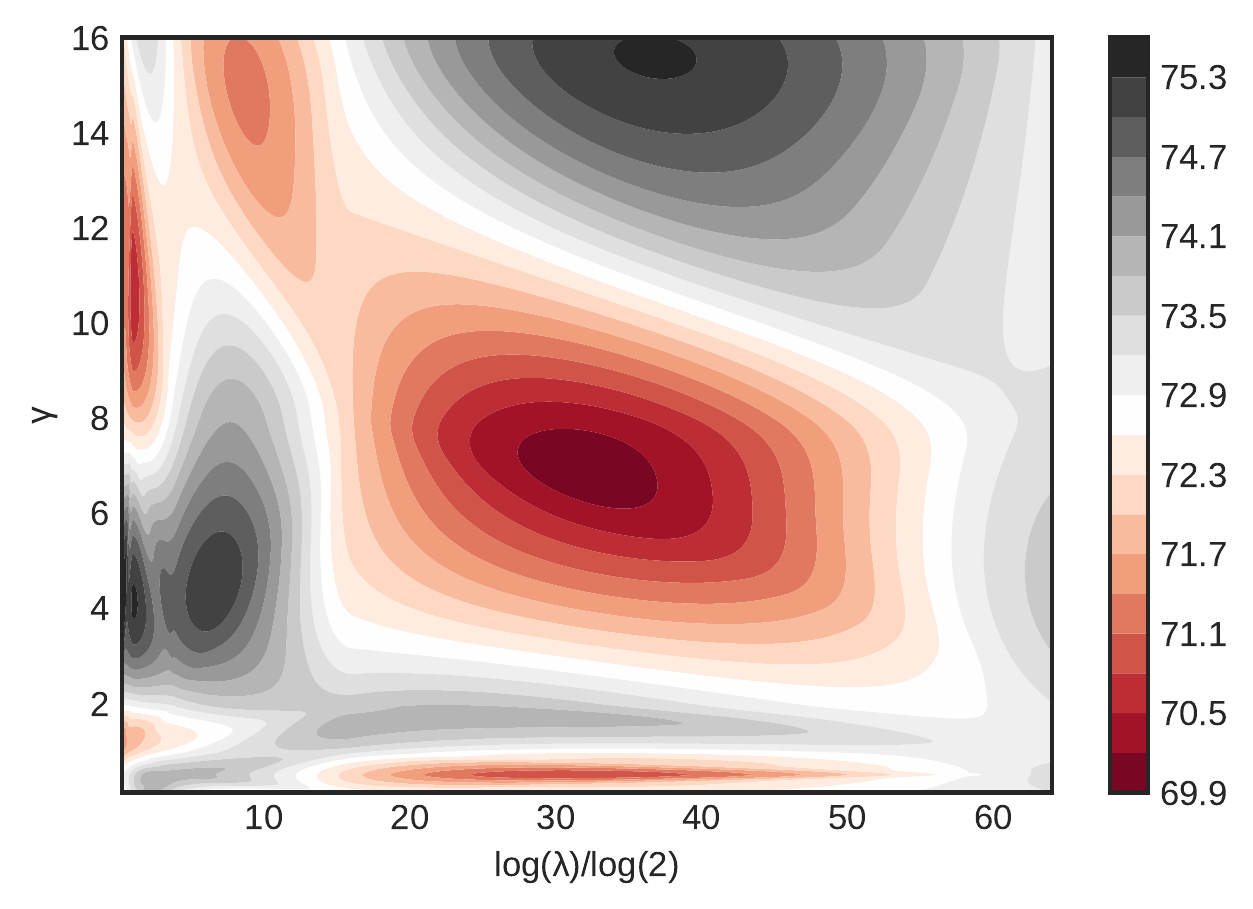} \\
        \end{tabular}
        \caption{Contour plots of the accuracy of the 1-NN on latent space  (left plot) and of the accuracy of the MTC (right plot) as functions of $\log_2\lambda$ and $\gamma$. The first row shows plots for Forest dataset, $n_1=28, n_2=14, k=5, \sigma=0.8$, the second row shows plots for SVHN dataset, $n_1=150, n_2=75, k=30, \sigma=0.8$.}
        \label{contour}
    \end{figure*}

\section{Related work}
The presented approach can be viewed as a method to regularize an autoencoder. There is a variety of  autoencoder regularization techniques, including higher-order contractive autoencoders~\cite{Rifai}, denoising autoencoders~\cite{Vincent}, variational autoencoders~\cite{kingma2014autoencoding} etc. 
Geometrical~\cite{gu2018learning,Skopek2020Mixed-curvature,chadebec2020geometryaware} and topological approaches~\cite{pmlr-v119-moor20a,schnenberger2020witness} to an autoencoder training have attracted some attention recently, though an emphasis has been made on a structure of the latent space (unlike the current paper, where we directly study the geometry of a data manifold). 

Using Ky-Fan $k$-antinorm in optimization problems for forcing a matrix to be of rank $k$ has been suggested
in~\cite{6389682} and further developed in~\cite{7182339,7335627,HONG2016216}.
\section{Conclusions and future work}
The idea of using autoencoders for the manifold learning task belongs to a data science folklore. Most of its realizations are defined by: a) an architecture of an autoencoder, b) an objective function, c) a method for an optimization (usually, the gradient descent). We employed the Constant Rank Theorem to substantiate the use of an additional term  forcing the Jacobian of an autoencoder to be of rank $k$. To minimize such an objective, we develop a new algorithm based on alternating the gradient descent and an update of tangent spaces (given by their basis vectors) to a learned manifold at some data points. Our experiments demonstrate that the designed algorithm successfully minimizes the objective and solves the manifold learning task. The quality of the manifold computed by the algorithm surpasses that of the Contractive Autoencoder which is the baseline with which we compare our method. The main practical limitation of our algorithm comes from a high cost of computing SVD at each iteration, which leads to poor scalability for  higher-dimensional data. Decreasing the parameter $M$ (a number of points at which the Ky-Fan antinorm of the Jacobian is minimized) significantly improves the speed of the algorithm at the expense of some deviation of an autoencoder's range from $k$-dimensionality. This problem can be partially solved by choosing new architectures of an autoencoder (e.g., by introducing some sparsity in its weights). 
Another limitation of our work is that we do not study how the Ky-Fan $k$-antinorm term affects a dataset representation in the latent space.
Future research will be focused on these issues. We will also study how an early termination and an optimization of parameters of the algorithm can be used to obtain a manifold faster and without the loss of its properties.

\section*{Acknowledgments}
The authors
would like to thank Atakan Varol  and Makat Tlebaliyev for providing computational resources of the Institute of Smart Systems and Artificial Intelligence (ISSAI).

\bibliographystyle{unsrt}  
\newcommand{\noopsort}[1]{}

\appendix
\section*{Proof of Theorem~\ref{opposite}}
W.l.o.g. let us assume $S = {\mathbb R}^n$.
We will use the following classical theorem.

\begin{theorem}[Constant Rank Theorem]
Let $A_0 \subseteq {\mathbb R}^n$ and $B_0 \subseteq {\mathbb R}^n$ be open sets,
$F: A_0\rightarrow  B_0$ be an infinitely differentiable mapping, and suppose the rank of $J_{F}$ on $A_0$ to be equal to $k$.
If $a\in A_0$ and $b = F(a)$, then there exist open sets $A \subseteq A_0$ and $B \subseteq B_0$ with
$a\in A$ and $F(A)\subseteq B$, and there exist infinitely differentiable one-to-one mappings $G: A\rightarrow U$, $H: B \rightarrow V$ (where $U, V\subseteq {\mathbb R}^n$ are open) such that $H\circ F\circ G^{-1} (U)\subseteq V$ and:
$$
H\circ F\circ G^{-1}(x_1, \cdots, x_n) = (x_1, \cdots, x_k, {\mathbf 0}_{n-k}).
$$
\end{theorem}

The Constant Rank Theorem implies that if ${\mathbf g}: {\mathbb R}^n \rightarrow {\mathbb R}^n$ is infinitely differentiable and ${\rm rank\,} J_{\mathbf g}({\mathbf x}) = k, {\mathbf x}\in {\mathbb R}^n$, then for any ${\mathbf p}\in {\mathbb R}^n$ there are open sets $A_{{\mathbf p}}\ni {\mathbf p}$, $B_{{\mathbf p}}\ni {\mathbf g}({\mathbf p})$ and infinitely differentiable one-to-one mappings $G_{{\mathbf p}}: A_{{\mathbf p}}\rightarrow U_{{\mathbf p}}$ and $H_{{\mathbf p}}: B_{{\mathbf p}}\rightarrow V_{{\mathbf p}}$ such that $$H_{{\mathbf p}}\circ {\mathbf g} \circ G_{{\mathbf p}}^{-1}(x_1, \cdots, x_n) = (x_1, \cdots, x_k, {\mathbf 0}_{n-k})$$
for any $(x_1, \cdots, x_n)\in U_{{\mathbf p}}$.
Therefore, $${\mathbf g} \circ G_{{\mathbf p}}^{-1}(x_1, \cdots, x_n) = H_{{\mathbf p}}^{-1}(x_1, \cdots, x_k, {\mathbf 0}_{n-k})$$ and ${\mathbf g} (A_{{\mathbf p}}) = H_{{\mathbf p}}^{-1}(W_{{\mathbf p}}\times \{{\mathbf 0}_{n-k}\})$ where
$$
W_{{\mathbf p}} = \{(x_1, \cdots, x_k)| (x_1, \cdots, x_k, {\mathbf 0}_{n-k})\in U_{{\mathbf p}}\}
$$
is an open subset of ${\mathbb R}^k$. It is easy to see that 
$$\phi_{{\mathbf p}}(x_1, \cdots, x_k) = H_{{\mathbf p}}^{-1}(x_1, \cdots, x_k, {\mathbf 0}_{n-k})$$ 
is a one-to-one infinitely differentiable mapping from $W_{{\mathbf p}}$ to ${\mathbf g} (A_{{\mathbf p}})$.
Thus, for a set ${\mathcal M} = {\mathbf g} ({\mathbb R}^n)$ and the topology $T$ one can deduce the following: A) for any ${\mathbf q}\in {\mathcal M}$, there is ${\mathbf p}\in {\mathbb R}^n$ such that ${\mathbf q}={\mathbf g}({\mathbf p})$; B) there is a neighbourhood ${\mathbf g} (A_{{\mathbf p}})\in T$ of ${\mathbf q}$, an open superset $B_{{\mathbf p}}\supseteq {\mathbf g} (A_{{\mathbf p}})$ and mappings $\phi = \phi_{{\mathbf p}}: W_{{\mathbf p}} \rightarrow {\mathbf g} (A_{{\mathbf p}})$, $\psi = \theta_n^k \circ H_{{\mathbf p}}: B_{{\mathbf p}}\rightarrow {\mathbb R}^k$, where $\theta_n^k(x_1, \cdots, x_n) = (x_1, \cdots, x_k)$, such that: a) $\phi$ is infinitely differentiable and one-to-one, b) $\psi$ is also infinitely differentiable and $\psi\circ \phi (x)=x$. Thus, ${\mathbf g} (A_{{\mathbf p}})$ is like ${\mathbb R}^k$. We conclude that there is a family of open sets in the topological space $({\mathbf g}({\mathbb R}^n), T[{\mathbf g}])$, i.e.
$$
\{{\mathbf g} (A_{{\mathbf p}})| {\mathbf p}\in {\mathbb R}^n\}\subseteq T[{\mathbf g}],
$$
such that any ${\mathbf p}\in {\mathbb R}^n$ has a neighbourhood ${\mathbf g} (A_{{\mathbf p}})$ that is like ${\mathbb R}^k$.

\section*{Proof of Theorem~\ref{locality}}
We need to use the Constant Rank Theorem. Since ${\rm rank\,} J_{{\mathbf g}}({\mathbf x}) = k$, then ${\rm rank\,} J_{{\mathbf g}}({\mathbf y}) = k$ for any ${\mathbf y}$ in some neighbourhood $U$ of ${\mathbf x}$.
Since ${\mathbf g}: U \rightarrow {\mathbb R}^n$ is infinitely differentiable, there are open sets $A_{{\mathbf x}}\ni {\mathbf x}$, $B_{{\mathbf x}'}\ni {\mathbf x}'$, $A_{{\mathbf x}}\subseteq U$ and infinitely differentiable one-to-one mappings $G: A_{{\mathbf x}}\rightarrow U_{{\mathbf x}}$ and $H: B_{{\mathbf x}'}\rightarrow V_{{\mathbf x}'}$ such that $$H\circ {\mathbf g} \circ G^{-1}(x_1, \cdots, x_n) = (x_1, \cdots, x_k, {\mathbf 0}_{n-k})$$
for any $(x_1, \cdots, x_n)\in U_{{\mathbf x}}$. Therefore, 
$${\mathbf g}(A_{{\mathbf x}}) = \{H^{-1}(x_1, \cdots, x_k, {\mathbf 0}_{n-k})| (x_1, \cdots, x_n)\in U_{{\mathbf x}}\}.$$
By construction, ${\mathbf g}(A_{{\mathbf x}})\subseteq \mathcal{M}\cap B_{{\mathbf x}'}$. 

The set ${\mathbf g}(A_{{\mathbf x}})$ is like ${\mathbb R}^k$, because $\phi: \theta_n^k(U_{{\mathbf x}})\to {\mathbf g}(A_{{\mathbf x}})$, where $$\phi(x_1, \cdots, x_k) = H^{-1}(x_1, \cdots, x_k, {\mathbf 0}_{n-k}),$$ is one-to-one and smooth. Its inverse $\phi^{-1}$ can be extended to a smooth function $\theta_n^k\circ H$ on $B_{{\mathbf x}'}$ because $\phi^{-1} = \theta_n^k\circ H|_{{\mathbf g}(A_{{\mathbf x}})}$.

For ${\mathbf x}'\in \mathcal{M}$ one can find an open set $V\ni {\mathbf x}'\subseteq {\mathbb R}^n$, such that $\mathcal{M}\cap V$ is like ${\mathbb R}^k$. Since $V$ is open, ${\mathbf g}(A_{{\mathbf x}})\cap V$ is also like ${\mathbb R}^k$. Analogously, from $B_{{\mathbf x}'}\subseteq {\mathbb R}^n$ being open we conclude that $\mathcal{M}\cap V\cap B_{{\mathbf x}'}$ is like ${\mathbb R}^k$ and
$$
{\mathbf g}(A_{{\mathbf x}})\cap V \subseteq \mathcal{M}\cap V\cap B_{{\mathbf x}'}.
$$
Using Lemma~\ref{inc-lemma} we conclude that there exists an open set $O\subseteq {\mathbb R}^n$ such that ${\mathbf g}(A_{{\mathbf x}})\cap V = \mathcal{M}\cap V\cap B_{{\mathbf x}'}\cap O$.  Since ${\mathbf g}(A_{{\mathbf x}})\cap V  = {\mathbf g}(A_{{\mathbf x}}\cap {\mathbf g}^{-1}(V))$ we finally obtain
$$
{\mathbf g}(\Omega) = \mathcal{M}\cap \Omega' ,
$$ 
where $\Omega = A_{{\mathbf x}}\cap {\mathbf g}^{-1}(V)$, $\Omega'=V\cap B_{{\mathbf x}'}\cap O$.
\begin{lemma}\label{inc-lemma} If subsets $A\subseteq B$ of ${\mathbb R}^n$ are both like ${\mathbb R}^k$, then there exists an open set $O\subseteq {\mathbb R}^n$ such that $A = B\cap O$. 
\end{lemma}
\begin{proof} By definition~\ref{like}, there are open sets $S_B\supseteq B, T_B\subseteq {\mathbb R}^k$ and a smooth one-to-one function $\phi_B: T_B\to B$ such that $\phi^{-1}_B$ can be extended to a smooth $\psi_B:S_B\to {\mathbb R}^k$, i.e. $\phi_B^{-1}(x)=\psi_B(x), x\in B$. For a set $A$, analogously $\phi_A, \psi_A, S_A, T_A$ are defined.

Let us prove that $\phi^{-1}_B(A)$ is open in ${\mathbb R}^k$. Indeed, since  $\phi_A: T_A\to A$ is one-to-one, then we conclude $\phi^{-1}_B(A) = \phi^{-1}_B(\phi_A(T_A)) = (\phi^{-1}_B\circ \phi_A)(T_A) =(\psi_B\circ \phi_A)(T_A)$. The  image of the open set $T_A$ is also open due to continuity and injectivity of $\psi_B\circ \phi_A: T_A\to {\mathbb R}^k$.  

Thus, $\psi^{-1}_B(\phi^{-1}_B(A)) = \psi^{-1}_B(\psi_B(A))$ is also open, and $O = \psi^{-1}_B(\psi_B(A))$ satisfies $O\cap B \subseteq A$ because for any $a\in B\setminus A$ we have $\psi_B(a)\notin \psi_B(A)$ (i.e. $a\notin \psi^{-1}_B(\psi_B(A))$). The opposite inclusion $O\cap B \supseteq A$ is obvious. Thus, $O\cap B = A$.
\end{proof}

\section*{Proof of Theorem~\ref{bound}}
Let us define $\boldsymbol{\gamma}_{{\mathbf u}}(t) = {\mathbf g}({\mathbf x}'+{\mathbf u}t)$ where ${\mathbf u}\in {\mathbb R}^n$ is such that $J_{{\mathbf g}} ({\mathbf x}'){\mathbf u}$ is a nonzero vector.

\begin{lemma}
If ${\mathbf x}$ satisfies the conditions of theorem, then
$$
\kappa ({\mathbf x}) \leq \sup_{{\mathbf u}: J_{{\mathbf g}} ({\mathbf x}'){\mathbf u}\ne {\mathbf 0}}\frac{\|\boldsymbol{\gamma}_{{\mathbf u}}'(0)\wedge \boldsymbol{\gamma}''_{{\mathbf u}}(0)\|}{\|\boldsymbol{\gamma}_{{\mathbf u}}'(0)\|^3},
$$
where $\wedge$ denotes the wedge product, i.e. $(\sum_{i=1}^n x_i {\mathbf e}_i)\wedge (\sum_{i=1}^n y_i {\mathbf e}_i) = \sum_{ij} x_i y_j ({\mathbf e}_i \wedge {\mathbf e}_j)$ where ${\mathbf e}_i \wedge {\mathbf e}_i={\mathbf 0}$, ${\mathbf e}_i \wedge {\mathbf e}_j = -{\mathbf e}_j \wedge {\mathbf e}_i$ and $\{{\mathbf e}_i \wedge {\mathbf e}_j|i< j\}$ is an orthonormal system of vectors. 
\end{lemma}
\begin{proof}
Recall that the Lagrange identity gives us $\|{\mathbf a} \wedge {\mathbf b}\|^2 = \|{\mathbf a}\|^2\|{\mathbf b}\|^2-({\mathbf a}\cdot {\mathbf b})^2$.

Let a vector-valued function $\boldsymbol{\gamma}^{{\mathbf u}}$ be the arc-length parameterization of the curve $\boldsymbol{\gamma}_{{\mathbf u}}$, i.e. $\boldsymbol{\gamma}^{{\mathbf u}}(x) = \boldsymbol{\gamma}_{{\mathbf u}}(s^{-1}(x))$ where $s(t)=\int_{0}^t \|\boldsymbol{\gamma}_{{\mathbf u}}'(x)\|dx$. Such a parameterization exists in some neighbourhood of $0$, because $s'(0) = \|\boldsymbol{\gamma}_{{\mathbf u}}'(0)\| = \|J_{{\mathbf g}} ({\mathbf x}'){\mathbf u}\|\ne 0$.

We have a standard derivation of the curve's curvature in terms of $\boldsymbol{\gamma}_{{\mathbf u}}$, i.e.
\begin{equation*}
\begin{split}
\|\frac{\partial^2 \boldsymbol{\gamma}^{{\mathbf u}}(0)}{\partial s^2}\| = \|\frac{\partial \big(\frac{\partial \boldsymbol{\gamma}^{{\mathbf u}}}{\partial s}\big)/\partial t}{\partial s/\partial t}(0)\| = 
\|\frac{\partial \big(\frac{\partial \boldsymbol{\gamma}_{{\mathbf u}}/\partial t}{\partial s/\partial t}\big)/\partial t}{\partial s/\partial t}(0)\| =\\ \|\frac{\frac{\boldsymbol{\gamma}_{{\mathbf u}}''(0) s'(0)-\boldsymbol{\gamma}_{{\mathbf u}}'(0)s''(0)}{s'(0)^2}}{s'(0)}\| = 
\frac{\|\boldsymbol{\gamma}_{{\mathbf u}}''(0) s'(0)-\boldsymbol{\gamma}_{{\mathbf u}}'(0)s''(0)\|}{s'(0)^3} .
\end{split}
\end{equation*}
Using $s'(x) = \|\boldsymbol{\gamma}'_{{\mathbf u}}(x)\|$ and $s''(x) = \frac{\boldsymbol{\gamma}_{{\mathbf u}}'(x)\cdot \boldsymbol{\gamma}_{{\mathbf u}}''(x)}{\|\boldsymbol{\gamma}_{{\mathbf u}}'(x)\|}$ we obtain
\begin{equation*}
\begin{split}
\frac{\|\boldsymbol{\gamma}_{{\mathbf u}}''(0) \|\boldsymbol{\gamma}'_{{\mathbf u}}(0)\|-\boldsymbol{\gamma}_{{\mathbf u}}'(0)\frac{\boldsymbol{\gamma}_{{\mathbf u}}'(0)\cdot \boldsymbol{\gamma}_{{\mathbf u}}''(0)}{\|\boldsymbol{\gamma}_{{\mathbf u}}'(0)\|}\|}{\|\boldsymbol{\gamma}'_{{\mathbf u}}(0)\|^3} = \\
\frac{\|\boldsymbol{\gamma}_{{\mathbf u}}''(0) (\boldsymbol{\gamma}'_{{\mathbf u}}(0)\cdot \boldsymbol{\gamma}'_{{\mathbf u}}(0)) -\boldsymbol{\gamma}_{{\mathbf u}}'(0)(\boldsymbol{\gamma}_{{\mathbf u}}'(0)\cdot \boldsymbol{\gamma}_{{\mathbf u}}''(0))\|}{\|\boldsymbol{\gamma}'_{{\mathbf u}}(0)\|^4}.
\end{split}
\end{equation*}
Direct computation gives $({\mathbf b} ({\mathbf a} \cdot {\mathbf a} )-{\mathbf a} ({\mathbf a} \cdot {\mathbf b} ))^2 =  \|{\mathbf a} \|^2 (\|{\mathbf a} \|^2\|{\mathbf b}\|^2-({\mathbf a} \cdot {\mathbf b})^2  ) = \|{\mathbf a} \|^2 \|({\mathbf a} \wedge {\mathbf b})\|^2$, and we have
$$
\|\frac{\partial^2 \boldsymbol{\gamma}^{{\mathbf u}}(0)}{\partial s^2}\| = \frac{\|\boldsymbol{\gamma}_{{\mathbf u}}'(0)\wedge \boldsymbol{\gamma}''_{{\mathbf u}}(0)\|}{\|\boldsymbol{\gamma}_{{\mathbf u}}'(0)\|^3}.
$$
If $\gamma:[-1,1]\to \mathcal{M}$ is a geodesic curve such that $\gamma'(0)={\mathbf u}$, then its curvature at ${\mathbf x}$ is bounded by the curvature of $\gamma_{{\mathbf u}}$ at ${\mathbf x}$, because both curves have the same direction (given by ${\mathbf u}$) at that point, but $\gamma$ has the lowest curvature among such curves on $\mathcal{M}$.
Thus, we proved
$$
\kappa ({\mathbf x}) \leq \sup_{{\mathbf u}: J_{{\mathbf g}} ({\mathbf x}'){\mathbf u}\ne {\mathbf 0}} \frac{\|\boldsymbol{\gamma}_{{\mathbf u}}'(0)\wedge \boldsymbol{\gamma}''_{{\mathbf u}}(0)\|}{\|\boldsymbol{\gamma}_{{\mathbf u}}'(0)\|^3}.
$$
\end{proof}
Using the lemma, it is now enough to prove
\begin{equation*}
\begin{split}
\sup_{{\mathbf u}: J_{{\mathbf g}} ({\mathbf x}'){\mathbf u}\ne {\mathbf 0}}\frac{\|\boldsymbol{\gamma}_{{\mathbf u}}'(0)\wedge \boldsymbol{\gamma}''_{{\mathbf u}}(0)\|}{\|\boldsymbol{\gamma}_{{\mathbf u}}'(0)\|^3} \leq 
\overline{\lim}_{\boldsymbol{\varepsilon}\to {\mathbf 0}}\min_{\lambda}\frac{\|[\boldsymbol{\varepsilon}^T H_{g_i}({\mathbf x}')\boldsymbol{\varepsilon}]_{i=1}^n-\lambda J_{{\mathbf g}} ({\mathbf x}')\boldsymbol{\varepsilon}\|}{\|J_{{\mathbf g}}\boldsymbol{\varepsilon}\|^2}.
\end{split}
\end{equation*}
Note that for any two vectors ${\mathbf a},  {\mathbf b}$ we have
\begin{equation*}
\begin{split}
\sqrt{({\mathbf a}\wedge {\mathbf b})\cdot ({\mathbf a}\wedge {\mathbf b})} = \sqrt{\|{\mathbf a}\|^2\|{\mathbf b}\|^2-({\mathbf a}\cdot {\mathbf b})^2} =
\|({\mathbf a}\wedge {\rm proj}_{\{{\mathbf a}\}^\perp}{\mathbf b})\| =  \\
\|{\mathbf a}\| \cdot \|{\rm proj}_{\{{\mathbf a}\}^\perp}{\mathbf b}\| = \min_{\lambda} \|{\mathbf a}\|\cdot \|{\mathbf b}-\lambda {\mathbf a}\|.
\end{split}
\end{equation*}
Let us set
$[{\mathbf a}]_i = \nabla g_i ({\mathbf x}'){\mathbf u}$ (i.e. ${\mathbf a} = J_{\mathbf g}{\mathbf u}$) and $[{\mathbf b}]_i = {\mathbf u}^T H_{g_i}({\mathbf x}'){\mathbf u}$. 
Thus, we obtain that the former supremum is
\begin{equation*}
\begin{split}
\sup_{{\mathbf u}: J_{{\mathbf g}} ({\mathbf x}'){\mathbf u}\ne {\mathbf 0}}\min_{\lambda}\frac{1}{\|J_{\mathbf g}{\mathbf u}\|^2}\|[{\mathbf u}^T H_{g_i}({\mathbf x}'){\mathbf u}- \lambda\nabla g_i ({\mathbf x}'){\mathbf u}]_{i=1}^n\| = \\
\sup_{{\mathbf u}: J_{{\mathbf g}} ({\mathbf x}'){\mathbf u}\ne {\mathbf 0}}\lim_{t\to 0}\min_{\lambda}\frac{1}{\|J_{\mathbf g}{\mathbf u}t\|^2}\|[{\mathbf u}^T H_{g_i}({\mathbf x}'){\mathbf u}t^2- \lambda\nabla g_i ({\mathbf x}'){\mathbf u}t]_{i=1}^n\|.
\end{split}
\end{equation*}
It remains to note
\begin{equation*}
\begin{split}
\sup_{{\mathbf u}: J_{{\mathbf g}} ({\mathbf x}'){\mathbf u}\ne {\mathbf 0}}\lim_{t\to 0}\min_{\lambda}\frac{1}{\|J_{\mathbf g}{\mathbf u}t\|^2}\|[{\mathbf u}^T H_{g_i}({\mathbf x}'){\mathbf u}t^2- 
\lambda\nabla g_i ({\mathbf x}'){\mathbf u}t]_{i=1}^n\| \leq \\
\overline{\lim}_{\boldsymbol{\varepsilon}\to {\mathbf 0}}\min_{\lambda}\frac{1}{\|J_{{\mathbf g}}\boldsymbol{\varepsilon}\|^2}\|[\boldsymbol{\varepsilon}^T H_{g_i}({\mathbf x}')\boldsymbol{\varepsilon}-\lambda \nabla g_i ({\mathbf x}')\boldsymbol{\varepsilon}]_{i=1}^n\|
\end{split}
\end{equation*}
and the bound is proved.
\begin{remark} The theorem's proof survives if ${\mathbf g}$ is defined on any open subset. Let us look to this bound for the function ${\mathbf g}({\mathbf x}) = \frac{{\mathbf x}}{\|{\mathbf x}\|}$, i.e. ${\mathbf g}: {\mathbb R}^n\setminus \{{\mathbf 0}\}\to S_1({\mathbf 0})$ where $S_1({\mathbf 0}) = \{{\mathbf x}\in {\mathbb R}^n | \, \|{\mathbf x}\|=1\}$. It is well-known that the curvature $\kappa=1$.

Direct computation gives $J_{{\mathbf g}} ({\mathbf x}) = \frac{I_n}{\|{\mathbf x}\|}-\frac{{\mathbf x}{\mathbf x}^T}{\|{\mathbf x}\|^3}$ and $J^2_{{\mathbf g}} ({\mathbf x}) = \frac{I_n}{\|{\mathbf x}\|^2}-\frac{{\mathbf x}{\mathbf x}^T}{\|{\mathbf x}\|^4} = \frac{J_{{\mathbf g}} ({\mathbf x})}{\|{\mathbf x}\|}$. The Hessian of every component is
\begin{equation*}
\begin{split}
H_{g_i} = J_{\frac{{\mathbf e}_i}{\|{\mathbf x}\|}-\frac{x_i{\mathbf x}}{\|{\mathbf x}\|^3}} = 
-\frac{{\mathbf e}_i{\mathbf x}^T}{\|{\mathbf x}\|^3}-\frac{{\mathbf x}{\mathbf e}_i^T}{\|{\mathbf x}\|^3}-x_i(\frac{I_n}{\|{\mathbf x}\|^3}-\frac{3{\mathbf x}{\mathbf x}^T}{\|{\mathbf x}\|^5}) .
\end{split}
\end{equation*}
Thus,
\begin{equation*}
\begin{split}
[\boldsymbol{\varepsilon}^T H_{g_i} \boldsymbol{\varepsilon}]_{i=1}^n = 
[-2\varepsilon_i\frac{{\mathbf x}^T\boldsymbol{\varepsilon}}{\|{\mathbf x}\|^3} - x_i\frac{\boldsymbol{\varepsilon}^T \boldsymbol{\varepsilon}}{\|{\mathbf x}\|^3}+x_i\frac{3({\mathbf x}^T\boldsymbol{\varepsilon})^2}{\|{\mathbf x}\|^5}]_{i=1}^n  = \\
-2\boldsymbol{\varepsilon}\frac{{\mathbf x}^T\boldsymbol{\varepsilon}}{\|{\mathbf x}\|^3} + {\mathbf x}(-\frac{\boldsymbol{\varepsilon}^T \boldsymbol{\varepsilon}}{\|{\mathbf x}\|^3}+\frac{3({\mathbf x}^T\boldsymbol{\varepsilon})^2}{\|{\mathbf x}\|^5}), \\
J_{{\mathbf g}} ({\mathbf x})\boldsymbol{\varepsilon} = \frac{\boldsymbol{\varepsilon}}{\|{\mathbf x}\|}-{\mathbf x}\frac{{\mathbf x}^T\boldsymbol{\varepsilon}}{\|{\mathbf x}\|^3}.
\end{split}
\end{equation*}
Let us denote ${\mathbf a} = [\boldsymbol{\varepsilon}^T H_{g_i} \boldsymbol{\varepsilon}]_{i=1}^n + 2\frac{{\mathbf x}^T\boldsymbol{\varepsilon}}{\|{\mathbf x}\|^2} J_{{\mathbf g}} ({\mathbf x})\boldsymbol{\varepsilon} $, i.e.
\begin{equation*}
\begin{split}
{\mathbf a} = {\mathbf x}(-\frac{\boldsymbol{\varepsilon}^T \boldsymbol{\varepsilon}}{\|{\mathbf x}\|^3}+\frac{({\mathbf x}^T\boldsymbol{\varepsilon})^2}{\|{\mathbf x}\|^5}) .
\end{split}
\end{equation*}
Therefore,
\begin{equation*}
\begin{split}
\|{\mathbf a}\| = |-\frac{\boldsymbol{\varepsilon}^T \boldsymbol{\varepsilon}}{\|{\mathbf x}\|^3}+\frac{({\mathbf x}^T\boldsymbol{\varepsilon})^2}{\|{\mathbf x}\|^5}| \cdot \|{\mathbf x}\| = \boldsymbol{\varepsilon}^T J^2_{{\mathbf g}} ({\mathbf x}) \boldsymbol{\varepsilon} \\
\end{split}
\end{equation*}
and we conclude 
\begin{equation*}
\begin{split}
\kappa({\mathbf x}) \leq 
\overline{\lim}_{\boldsymbol{\varepsilon}\to {\mathbf 0}} \min_{\lambda}\frac{\|[\boldsymbol{\varepsilon}^T H_{g_i} \boldsymbol{\varepsilon}]_{i=1}^n-\lambda J_{{\mathbf g}} ({\mathbf x}')\boldsymbol{\varepsilon}\|}{\|J_{{\mathbf g}}\boldsymbol{\varepsilon}\|^2}  \leq 
\frac{\|{\mathbf a}\|}{\boldsymbol{\varepsilon}^T J^2_{{\mathbf g}} ({\mathbf x}) \boldsymbol{\varepsilon}}=1.
\end{split}
\end{equation*}
In other words, the upper bound of Theorem~\ref{bound} equals the curvature at each point.
\end{remark}

\section*{Proof of Theorem~\ref{encoder-decoder}}

First we bound:
\begin{equation*}
\begin{split}
\min_{\lambda}\frac{\|[\boldsymbol{\varepsilon}^T H_{g_i}({\mathbf x}')\boldsymbol{\varepsilon}]_{i=1}^n-\lambda J_{{\mathbf g}} ({\mathbf x}')\boldsymbol{\varepsilon}\|}{\|J_{{\mathbf g}}({\mathbf x}')\boldsymbol{\varepsilon}\|^2} = \min_{\lambda}\frac{\|[\boldsymbol{\varepsilon}^T H_{g_i}({\mathbf x}')\boldsymbol{\varepsilon}]_{i=1}^n-\lambda J_{{\mathbf g}} ({\mathbf x}')\boldsymbol{\varepsilon}\|}{\|J_{{\mathbf d}}J_{{\mathbf e}}\boldsymbol{\varepsilon}\|^2} \leq \\
\min_{\lambda}\frac{\|[\boldsymbol{\varepsilon}^T H_{g_i}({\mathbf x}')\boldsymbol{\varepsilon}]_{i=1}^n-\lambda J_{{\mathbf g}} ({\mathbf x}')\boldsymbol{\varepsilon}\|}{\sigma^2_k\|J_{{\mathbf e}}\boldsymbol{\varepsilon}\|^2} .
\end{split}
\end{equation*}
Then, for any $f=d_i$ one can express the Jacobian and the Hessian of $f\circ {\mathbf e}$ as
\begin{equation*}
\begin{split}
J_{f\circ {\mathbf e}} \boldsymbol{\varepsilon} = \sum_{i,j} \partial_i f \partial_j e_i \varepsilon_j \\
\boldsymbol{\varepsilon}^T H_{f\circ {\mathbf e}}\boldsymbol{\varepsilon} = \sum_{i,j,k}\partial_i f \partial_{j,k} e_i(x)\varepsilon_j \varepsilon_k+
\sum_{i,j,k,l} \partial_{i,l} f\partial_{k}e_l \partial_j e_i \varepsilon_j\varepsilon_k = \\
\sum_{i}\partial_i f  \boldsymbol{\varepsilon}^T H_{e_i}\boldsymbol{\varepsilon}+\sum_{i,l}\partial_{i,l} f (J_{{\mathbf e}} \boldsymbol{\varepsilon})_i (J_{{\mathbf e}} \boldsymbol{\varepsilon})_l = 
\nabla f  [\boldsymbol{\varepsilon}^T H_{e_i}\boldsymbol{\varepsilon}]_{i=1}^k+ H_{f}\cdot J_{{\mathbf e}} \boldsymbol{\varepsilon}\boldsymbol{\varepsilon}^T J_{{\mathbf e}}^T.
\end{split}
\end{equation*}
Thus, we have
$$
[\boldsymbol{\varepsilon}^T H_{g_i}({\mathbf x}')\boldsymbol{\varepsilon}]_{i=1}^n = J_{\mathbf d}[\boldsymbol{\varepsilon}^T H_{e_i}\boldsymbol{\varepsilon}]_{i=1}^k +[H_{d_i}\cdot J_{{\mathbf e}} \boldsymbol{\varepsilon}\boldsymbol{\varepsilon}^T J_{{\mathbf e}}^T]_{i=1}^n,
$$
where $A\cdot B = \sum_{ij}A_{ij}B_{ij}$.
Let us denote $\frac{\|[\boldsymbol{\varepsilon}^T H_{e_i}\boldsymbol{\varepsilon}]_{i=1}^k \wedge J_{\mathbf e}\boldsymbol{\varepsilon}\|}{\|J_{\mathbf e}\boldsymbol{\varepsilon}\|^3}$ by $C$. Then,
\begin{equation*}
\begin{split}
C\|J_{\mathbf e}\boldsymbol{\varepsilon}\|^2  = \frac{\|[\boldsymbol{\varepsilon}^T H_{e_i}\boldsymbol{\varepsilon}]_{i=1}^k \wedge J_{\mathbf e}\boldsymbol{\varepsilon}\|}{\|J_{\mathbf e}\boldsymbol{\varepsilon}\|}, \\
\min_{\lambda}\|[\boldsymbol{\varepsilon}^T H_{e_i}\boldsymbol{\varepsilon}]_{i=1}^k -\lambda J_{\mathbf e}\boldsymbol{\varepsilon}\|  =\|[\boldsymbol{\varepsilon}^T H_{e_i}\boldsymbol{\varepsilon}]_{i=1}^k -\lambda^\ast J_{\mathbf e}\boldsymbol{\varepsilon}\| 
\end{split}
\end{equation*}
for some $\lambda^\ast$.
Finally we obtain
\begin{equation*}
\begin{split}
\min_{\lambda}\|J_{\mathbf d}[\boldsymbol{\varepsilon}^T H_{e_i}\boldsymbol{\varepsilon}]_{i=1}^k +[H_{d_i}\cdot J_{{\mathbf e}} \boldsymbol{\varepsilon}\boldsymbol{\varepsilon}^T J_{{\mathbf e}}^T]_{i=1}^n-
\lambda J_{{\mathbf d}} J_{{\mathbf e}}\boldsymbol{\varepsilon}\| \leq \\ 
\|J_{\mathbf d}([\boldsymbol{\varepsilon}^T H_{e_i}\boldsymbol{\varepsilon}]_{i=1}^k - \lambda^\ast J_{{\mathbf e}}\boldsymbol{\varepsilon}) +[H_{d_i}\cdot J_{{\mathbf e}} \boldsymbol{\varepsilon}\boldsymbol{\varepsilon}^T J_{{\mathbf e}}^T]_{i=1}^n\|\leq \\
\sigma_{1}\|[\boldsymbol{\varepsilon}^T H_{e_i}\boldsymbol{\varepsilon}]_{i=1}^k - \lambda^\ast J_{{\mathbf e}}\boldsymbol{\varepsilon}\|+\|[H_{d_i}\cdot J_{{\mathbf e}} \boldsymbol{\varepsilon}\boldsymbol{\varepsilon}^T J_{{\mathbf e}}^T]_{i=1}^n\| \leq \\
\sigma_1 C\|J_{\mathbf e}\boldsymbol{\varepsilon}\|^2 +\sqrt{\sum_{i}\|H_{d_i}\|_F^2 \|J_{{\mathbf e}} \boldsymbol{\varepsilon}\boldsymbol{\varepsilon}^T J_{{\mathbf e}}^T\|_F^2} =  
\|J_{\mathbf e}\boldsymbol{\varepsilon}\|^2 (\sigma_1 C+\sqrt{\sum_{i}\|H_{d_i}\|_F^2 }).
\end{split}
\end{equation*}
Therefore, $\min_{\lambda}\frac{\|[\boldsymbol{\varepsilon}^T H_{g_i}({\mathbf x}')\boldsymbol{\varepsilon}]_{i=1}^n-\lambda J_{{\mathbf g}} ({\mathbf x}')\boldsymbol{\varepsilon}\|}{\|J_{{\mathbf g}}({\mathbf x}')\boldsymbol{\varepsilon}\|^2} $ is bounded by
$$
 \frac{\sigma_1 C}{\sigma^2_k} + \frac{\sqrt{\sum_{i}\|H_{d_i}\|_F^2 }}{\sigma^2_k}.
$$
Theorem proved.

\section*{Details of experiments}
Hyperparameters in experiments of Section~\ref{rwd} include (a) the number of neurons on the first and second layers of the encoder, denoted by $n_1$ and $n_2$, that by the encoder-decoder symmetry define the number of neurons on the first ($n_2$) and second ($n_1$) layers of a decoder; (b) the curvature regularization parameter $\gamma$; (c) the rank regularization parameter $\lambda$; (d) the variance of added noise $\sigma^2$; (e) the number of iterations of the Algorithm~\ref{alternate}, $T$; (f) the number of points at which the Ky-Fan antinorm of the gradient is penalized, $M$; (g) parameters of Adam Optimizer $\beta_1, \beta_2$ and the learning rate $\alpha$; (h) the batch size $m$ ;(i) the parameter $\beta$ of MTC that controls weight of the tangent propagation term; (j) the parameter $K$ in the $K$-NN. For all datasets parameters (except for $n_1$ and $n_2$) were set according to Table~\ref{hyperparameters}. In all experiments we define $n_2=\frac{n_1}{2}$.
Finally, for each dataset, parameters $n_1$ and $k$ (the number $k$ in the Ky-Fan antinorm) were defined separately: (c) for MNIST, $n_1\in \{300, 400\}$ and $k\in \{30, 60, 90\}$; (d) for CIFAR10, $n_1\in \{400, 500\}$ and $k\in \{50, 100, 150\}$; (e) for SVHN, $n_1\in \{180, 240\}$ and $k\in \{30, 60, 90\}$; (f) for STL-10, $n_1\in \{300, 500\}$ and $k\in \{50, 75, 100\}$. In CAE+H method the same architecture of an autoencoder was used as in the Algorithm~\ref{alternate}. 

\begin{table}
\centering
\begin{tabular}{|c|c|} 
\hline
$\gamma=0.5$, $\sigma=0.8$,\\
$\lambda \in \{10, 40\}$,\\
$T=1000$, $M=1000$, $m =20$,\\
$\beta_1=0.9, \beta_2=0.999, \alpha=0.001$, \\
$\beta \in \{0.0, 0.1, 0.01\}$, \\
$K\in \{1,2,\cdots, 19\}$ \\
\hline
\end{tabular}
\caption{Values of hyperparameters.}\label{hyperparameters}
\end{table}
\end{document}